\documentclass[conference]{IEEEtran}
\IEEEoverridecommandlockouts
\usepackage{cite}
\usepackage{amsmath,amssymb,amsfonts}

\usepackage{amsthm}
\usepackage{graphicx}
\usepackage{textcomp}
\usepackage{hyperref}
\usepackage{amsmath}
\usepackage{algorithmicx}
\usepackage[ruled]{algorithm}
\usepackage{algpseudocode}
\usepackage{algpascal}
\usepackage{algc}
\usepackage{color}

\newcommand{\bfx}{\boldsymbol{x}}

\usepackage{amssymb,amsmath,graphicx}
\usepackage[english]{babel}
\def\tr{{\raise0pt\hbox{$\scriptscriptstyle\top$}}}

\newtheorem{example}{Example}
\newtheorem{proposition}{Proposition}
\newtheorem{theorem}{Theorem}
\newtheorem{remark}{Remark}
\newtheorem{corollary}{Corollary}
\newtheorem{definition}{Definition}

\def\eop{\hfill{$\Box$}\medskip}

\def\BibTeX{{\rm B\kern-.05em{\sc i\kern-.025em b}\kern-.08em
    T\kern-.1667em\lower.7ex\hbox{E}\kern-.125emX}}
    
\definecolor{redcolor}{rgb}{0.7,0.3,0.3}
\definecolor{bluecolor}{rgb}{0.1,0.1,0.8}

\begin{document}

\title{Augmented Artificial Intelligence:\\ a Conceptual Framework
\thanks{ANG and IYT were Supported by Innovate UK (KTP009890 and KTP010522) and Ministry of science and education, Russia (14.Y26.31.0022). BG  thanks the University of Leicester for granting him academic study leave to do this research. }
}

\author{\IEEEauthorblockN{1\textsuperscript{st} Alexander N Gorban}
\IEEEauthorblockA{\textit{University of Leicester} \\
\textit{and Lobachevsky University}\\
Leicester, UK, and Nizhni Novgorod, Russia \\
a.n.gorban@le.ac.uk}
\and
\IEEEauthorblockN{2\textsuperscript{nd} Bogdan Grechuk}
\IEEEauthorblockA{\textit{University of Leicester} \\
Leicester, UK\\
bg83@le.ac.uk}
\and
\IEEEauthorblockN{3\textsuperscript{rd} Ivan Y Tyukin}
\IEEEauthorblockA{\textit{University of Leicester} \\
\textit{and Lobachevsky University}\\
Leicester, UK, and Nizhni Novgorod, Russia \\
i.tyukin@le.ac.uk}
}

\maketitle

\begin{abstract}
All artificial Intelligence (AI) systems make errors. These errors are unexpected,
and differ often from the typical human mistakes (``non-human'' errors). The AI errors
should be corrected  without damage of existing skills and, hopefully, avoiding direct human expertise.  This  {paper} presents an initial summary report of project taking new and
systematic approach to improving the intellectual effectiveness of the
individual AI by communities of AIs. We combine some  ideas of learning in heterogeneous
 multiagent systems with new and original mathematical approaches for non-iterative corrections
 of errors of legacy AI systems.  The  mathematical foundations of AI non-destructive correction are presented and a series of new stochastic separation theorems is proven. These theorems  provide a new instrument for the development, analysis, and assessment of machine learning methods and algorithms in high dimension.  They demonstrate that in high dimensions and even for exponentially large samples, linear classifiers in their classical Fisher's form are powerful enough to separate errors from correct responses with high probability and to provide efficient solution to the non-destructive corrector problem. In particular, we prove some hypotheses formulated in our paper `Stochastic Separation Theorems' (Neural Networks, 94, 255--259,  2017), and answer one general problem published by Donoho and Tanner in 2009. 
\end{abstract}

\begin{IEEEkeywords}
multiscale experts, knowledge transfer, non-iterative learning, error correction, measure concentration, blessing of dimensionality
\end{IEEEkeywords}

\section{Introduction}
The history of neural networks research can be represented as a series of explosions or waves of inventions and expectations. This history ensures us that the popular Gartner's hype cycle for emerging technologies presented by the solid curve on Fig.~\ref{Fig:SlideGartnerInnovationCycle} (see, for example \cite{ColumbusGartner}) should be supplemented by the new peak of expectation explosion (dashed line). Some expectations from the previous peak are realized and move to the ``Plateau of Productivity'' but the majority of them jump to the next ``Peak of Inflated Expectations''. This observation relates not only to neural technologies but perhaps to majority of IT innovations. It is surprising to see, how expectations reappear in the new wave from the previous peak often without modifications, just changing the human carriers.
\begin{figure}[ht]
\centering
\includegraphics[width=0.95\columnwidth]{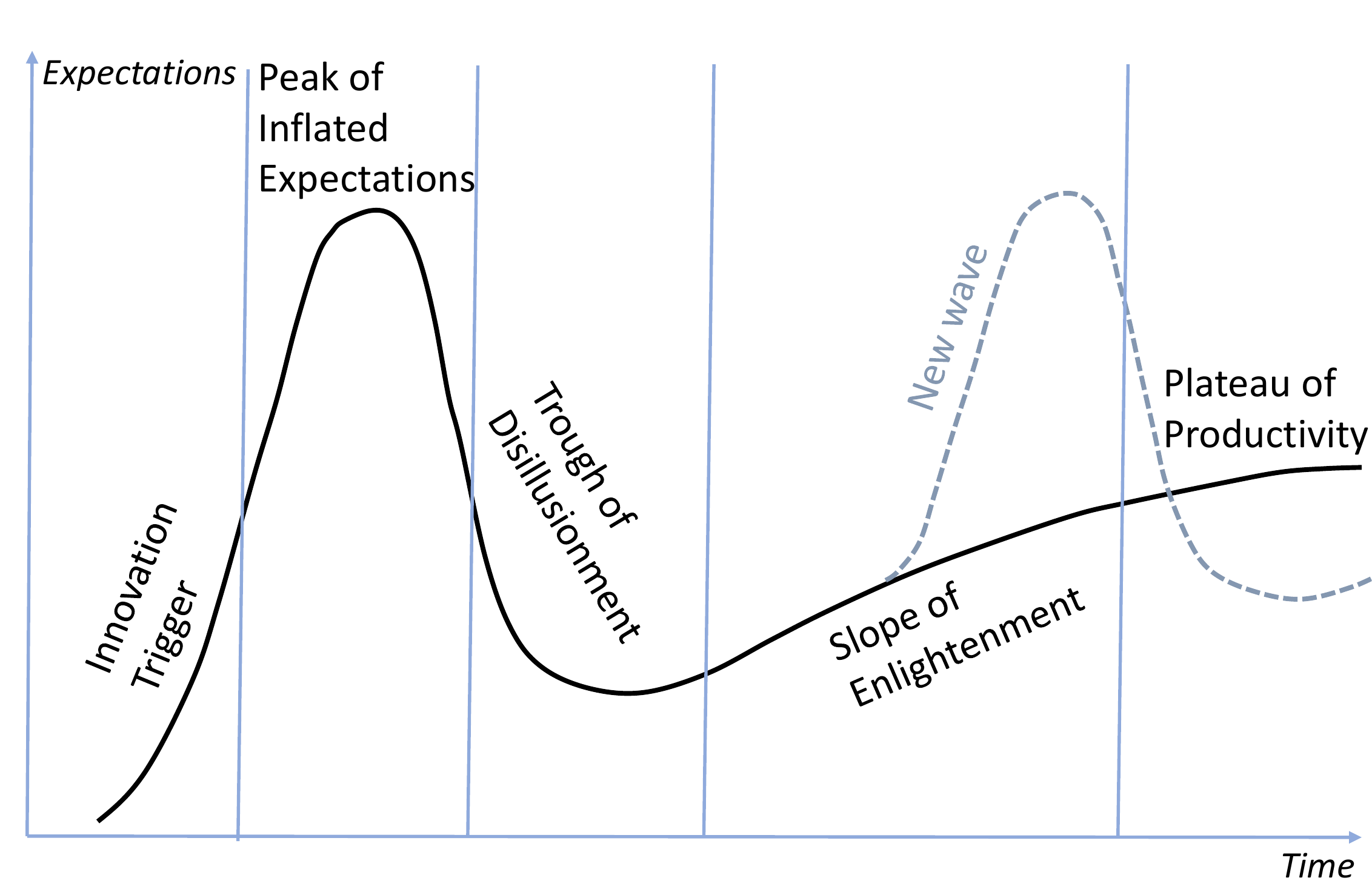}
\caption {Gartner's Hype Cycle for emerging technologies supplemented by a new peak.}
\label{Fig:SlideGartnerInnovationCycle}
\end{figure}

Computers and networks have been expected to augment the human intelligence \cite{Engelbart1962}. In 1998 one of the authors had been inspired by 8 years of success of knowledge discovery by deep learning neural network and by the transformation of their hidden knowledge into explicit knowledge in the form of ``logically transparent networks'' \cite{GorbanMirTsar1999} by means of pruning, binarization and other simplification procedures \cite{Gorban1990,LeCunBrain1990}, and wrote: ``I am sure that the neural network technology of knowledge discovery is a "point of growth", which will remodel  neuroinformatics, transform many areas of information technologies and create new approaches'' \cite{Gorban1998}. Now it seems that this prediction will not be fulfilled: most customers do not care about gaining knowledge but  prefer  the ``one button solutions'', which exclude humans from the process as far as it is possible. This is not a new situation in history. New intellectual technologies increase intellectual abilities of mankind, but not the knowledge of individual humans. Here, we can refer to Plato ``There is an old Egyptian tale of Theuth, the inventor of  writing, showing his invention to the god Thamus, who told him that he would only spoil men's memories and take away their understandings'' \cite{Plato}. The adequate model of future Artificial Intelligence (AI) usage should include large communities of AI systems. Knowledge should circulate and grow in these communities. Participation of humans in these processes should be minimized. In the course of this technical revolution  not the ``Augmented human intellect''   but the continuously augmenting AI will be created.

In this work, we propose the conceptual framework for augmenting AI in communities or ``social networks'' of AIs. For construction of such social networks, we employ several ideas in addition to the classical machine learning. The first of them is separation of the problem areas between small local (neural)  experts, their competitive and collaborative learning, and  conflict resolution. In 1991, two first papers with this idea were published simultaneously \cite{GilGorbanMir1991,JacobsHinton1991}. The techniques for distribution of tasks between small local experts were developed. In our version of this technology \cite{GilGorbanMir1991}  and in all our applied software \cite{GorbanMirTsar1999,Multineuron1994,Multineuron1995n1,Multineuron1995n2} the neural network answers were always complemented by the evaluation of the network self-confidence. This self-confidence level  is an important instrument for community learning.

The second idea is the blessing of dimensionality \cite{Kainen1997,Donoho2000,AndersonEtAl2014,GorTyuRom2016,GorTyukPhil2018} and  the AI correction method \cite{GorbanRomBurtTyu2016} based on stochastic  separation theorems  \cite{GorbTyu2017}. The ``sparsity'' of high-dimensional spaces and  concentration of measure phenomena make some low-dimensional approaches impossible in high dimensions. This problem is widely known as the ``curse of dimensionality''. Surprisingly, the same phenomena can be efficiently employed for creation of new, high-dimensional methods, which seem to be much simpler than in low dimensions. This is the ``blessing of dimensionality''.

The classical theorems about concentration of measure  state that random points in a highly-dimensional data distribution are concentrated in a thin layer near an average or median level set of a Lipschitz function (for introduction into this area we refer to \cite{Ledoux2005}). The newly discovered stochastic separation theorems \cite 2017 revealed the fine structure of these thin layers: the random  points  are all linearly separable from the rest of the set even for exponentially large random sets. Of course, the probability distribution should be `genuinely' high-dimensional for all these concentration and separation theorems.

Linear separability of exponentially large random subsets in high dimension allows us to solve the problem of nondestructive correction of legacy AI systems: the linear classifiers in their simplest Fisher’s form can separate errors from correct responses with high probability  \cite{GorbanRomBurtTyu2016}. It is possible to avoid the standard assumption about independence and identical distribution of data points (i.i.d.). The non-iterative and nondestructive  correctors  can be employed for skills transfer in communities of AI systems \cite{Tyukin2017a}.

These two ideas are joined in a special organisational environment of community learning
which is organized in several phases:
\begin{itemize}
\item Initial supervising learning where community of newborn experts assimilate the knowledge  hidden in  labeled tasks from a problem-book (the problem-book is a continuously growing and transforming collection of samples);
\item Non-iterative learning of community with self-labeling of real-life or additional training samples on the basis of separation of expertise between local experts, their continuous adaptation and mutual correction for the assimilation of gradual  changes in reality.
\item Interiorisaton of the results of the self-supervising learning of community in the internal skills of experts.
\item Development and learning of special network manager that evaluates the level of expertise of the local experts  for a problem and distributes the incoming  task flow  between them.
\item Using an ``ultimate auditor'' to assimilate   qualitative changes in the environment and correct   collective errors; it may be human inspection, a  feedback from real life, or another system of interference into the self-labeling process.
\end{itemize}

We describe the main constructions of this approach using the example of classification problems and simple linear correctors. The correctors with higher abilities can be constructed on the basis of small neural networks with uncorrelated neurons \cite{GorbanRomBurtTyu2016} but already single-neuron correctors (Fisher's discriminants) can help in explanation of a wealth of empirical evidence related to in-vivo recordings of ``Grandmother'' cells and ``concept'' cells \cite{GorTyukPhil2018,TyukinBrain2017}. We pay special attention to the mathematical backgrounds of the technology and prove a series of new stochastic separation theorems. In particular, we find that the hypothesis about stochastic separability for general log-concave distribution \cite{Tyukin2017a} is true and describe  a general class of probability measures with linear stochastic separability, the question was asked in 2009 by Donoho and Tanner \cite{DonohoTanner2009}.

\section{Supervising stage: Problem Owners, Margins, Self-confidence, and Error functions}

\begin{figure}
\centering
\includegraphics[width=0.95\columnwidth]{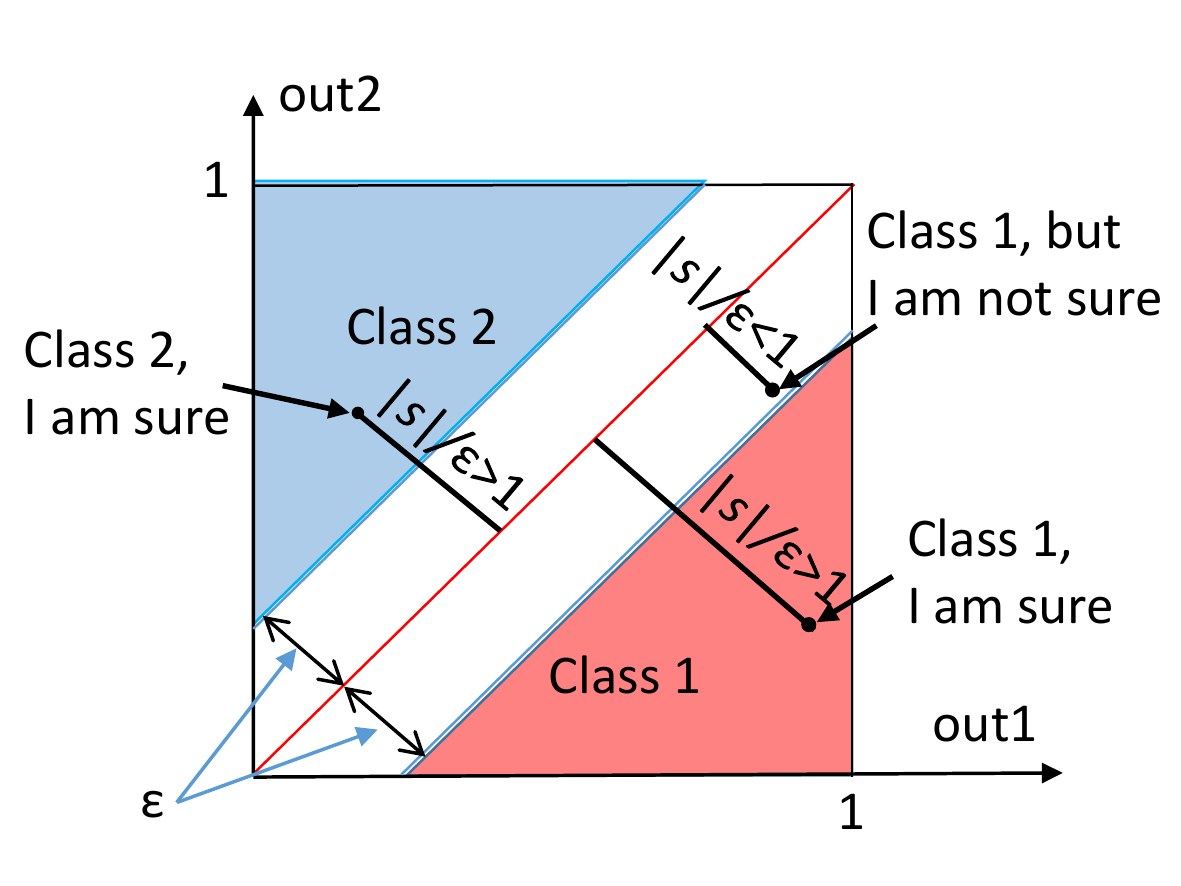}
\caption {Answers and assurance; $s$ is the deviation from the diagonal.}
\label{Fig:SoftMarginEstime}
\end{figure}

Consider binary classification problems. The neural experts with arbitrary internal structure have two outputs, out1 and out2, with interpretation: the sample belongs to class 1 if out1$\geq$out2 and it belongs to class 2 if out1$<$out2. For any given $\varepsilon>0$ we can define the level of (self-)confidence in the classification answer as it demonstrated in Fig.~\ref{Fig:SoftMarginEstime}. The {\em owner of a sample} is an expert that gives the best (correct and most confident) answer for this sample. If we assume the single owner for every sample then in the community functioning for problem solving this single owner gives the final result (Fig.~\ref{Fig:WinnerTA}).

\begin{figure}
\centering
\includegraphics[width=0.95\columnwidth]{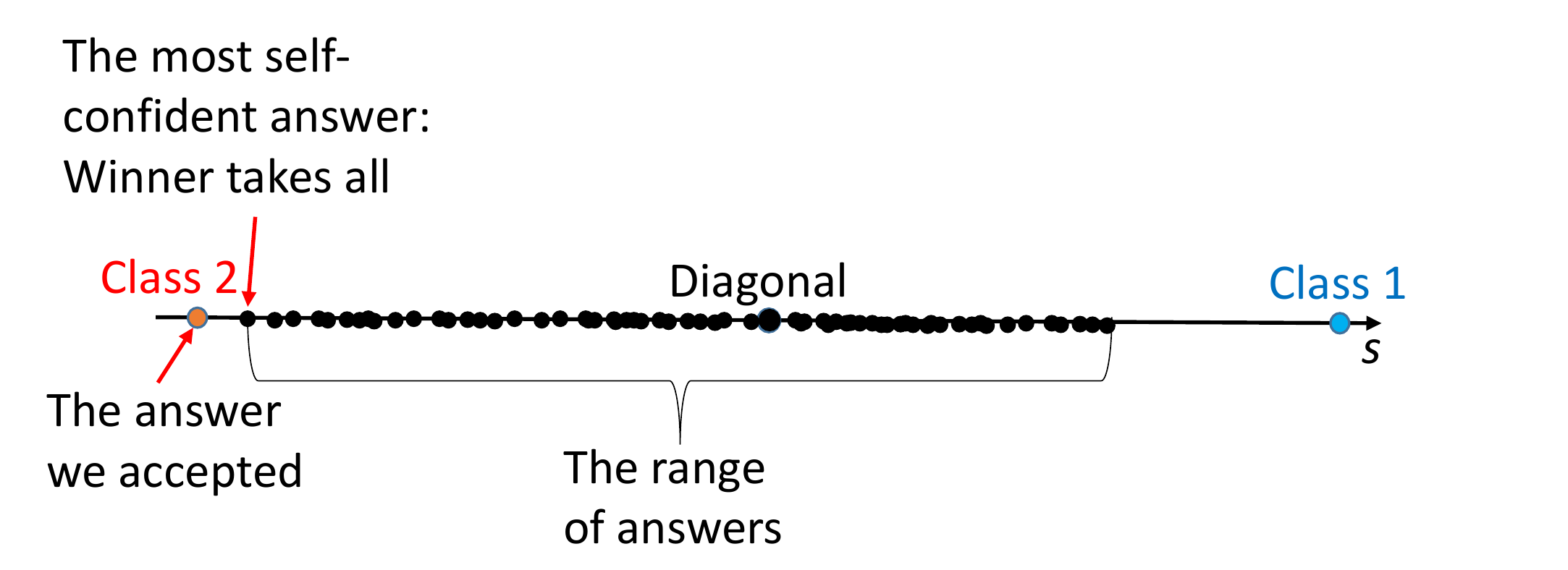}
\caption {Interpretation of community answer: Most self-confident winner takes all. Dots correspond to the various agents' answers, $s$ is defined in Fig.~\ref{Fig:SoftMarginEstime}.}
\label{Fig:WinnerTA}
\end{figure}
We aim to train the community of agents in such a way that they will give  correct self-confident answers to the samples they own, and do not make large mistakes on all other examples they never met before. The desired histogram of answers is presented in Fig.~\ref{Fig:WinnerTA}.
\begin{figure}
\centering
\includegraphics[width=0.95\columnwidth]{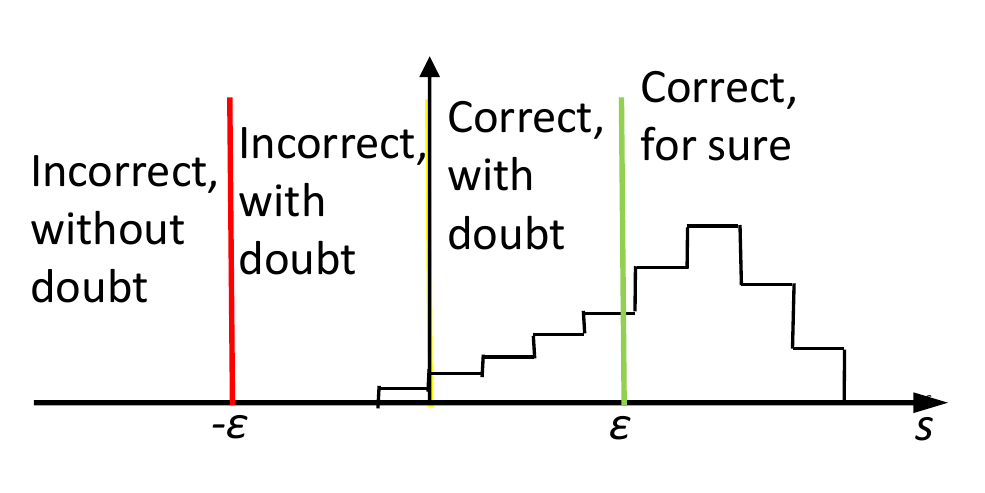}
\caption {Histogram of answers for trained community of agents: they should give a correct self-confident answer to the samples they own, and do not make large mistakes on all other examples they never met before.}
\label{Fig:ValidationHYstogram}
\end{figure}

Learning is minimisation of error functionals, which is defined for any selected sample and any local expert. This error function should be different for owners and non-ofners of the sample. If we assume that each smalpe has a single owner then the  error function presented in Fig.~\ref{Fig:OwnerNonOwner} can be used.
\begin{figure}
\centering
\includegraphics[width=0.95\columnwidth]{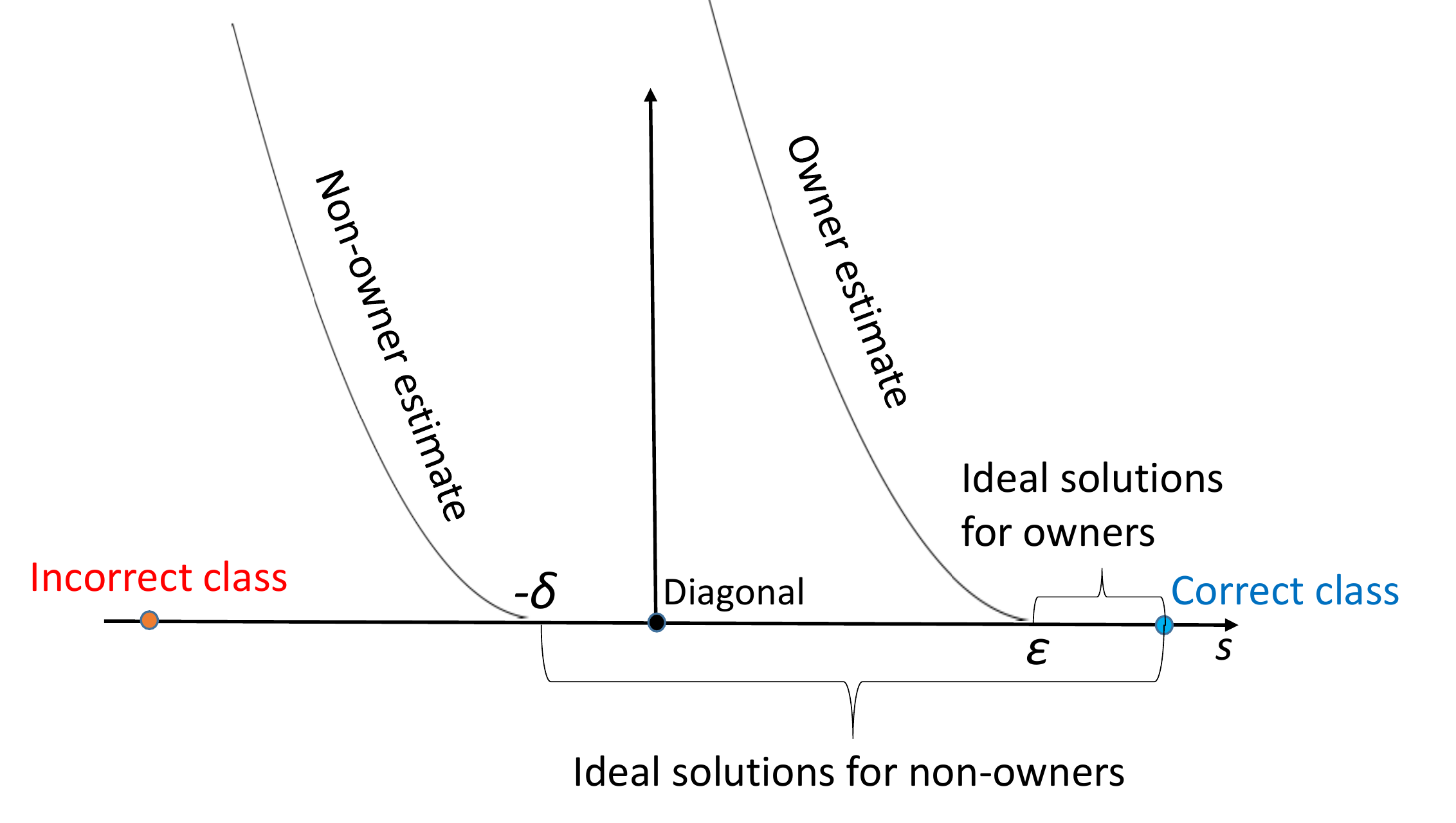}
\caption {Soft margin error function for owners and non-owners (one owner).}
\label{Fig:OwnerNonOwner}
\end{figure}

\begin{figure}
\centering
\includegraphics[width=0.95\columnwidth]{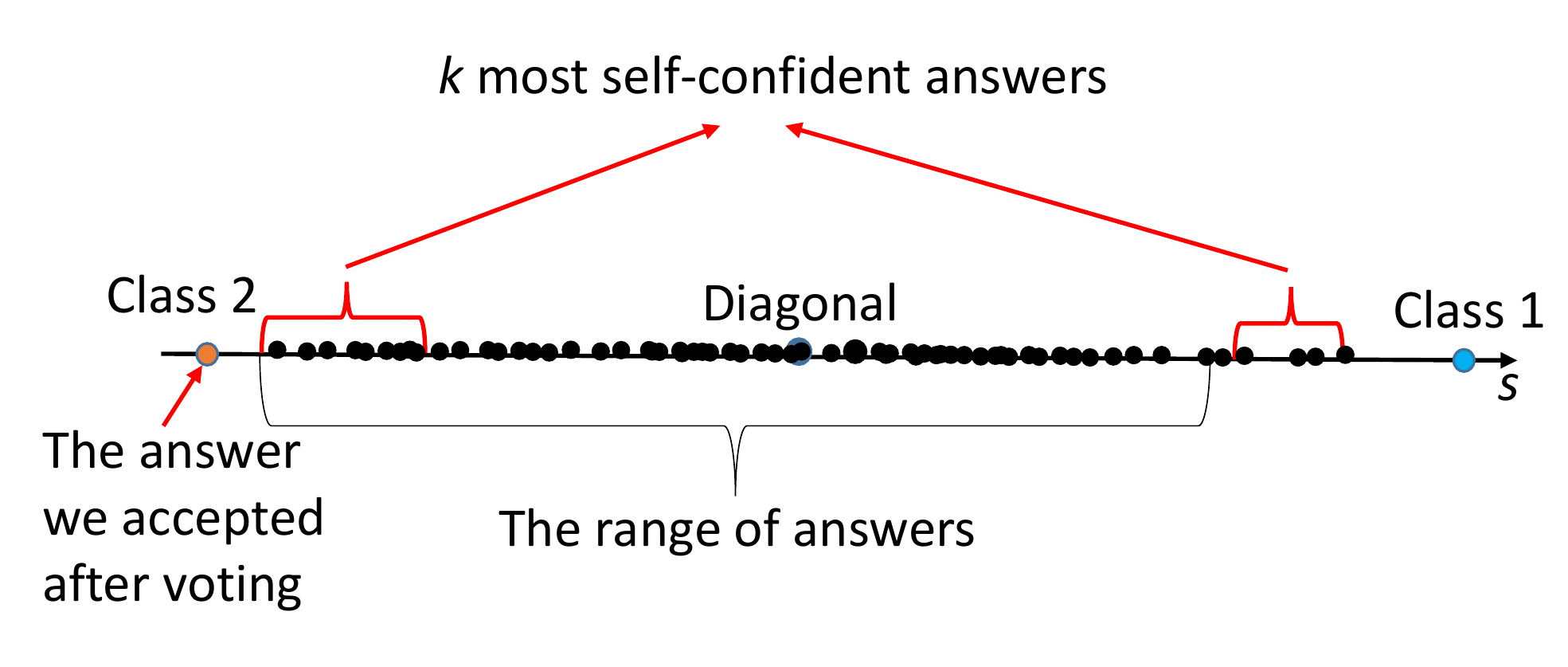}
\caption {Interperation of community answer with collective ownership: Voting of $k$ most self-confident winners.}
\label{Fig:Voting}
\end{figure}
Voting of $k$ most self-confident experts (Fig.~\ref{Fig:Voting}) can make the decision more stable. This voting may be organised with  weights of votes, which depend on the individual experts' level of confidence, or without weights, just as a simple voting.   The modified error function for system with collective ownership (each sample has $k$ owners)  is needed (Fig.~\ref{Fig:kWinnersTA}). This function is constructed to provide proper answers of all $k$ owners.

\begin{figure}
\centering
\includegraphics[width=0.95\columnwidth]{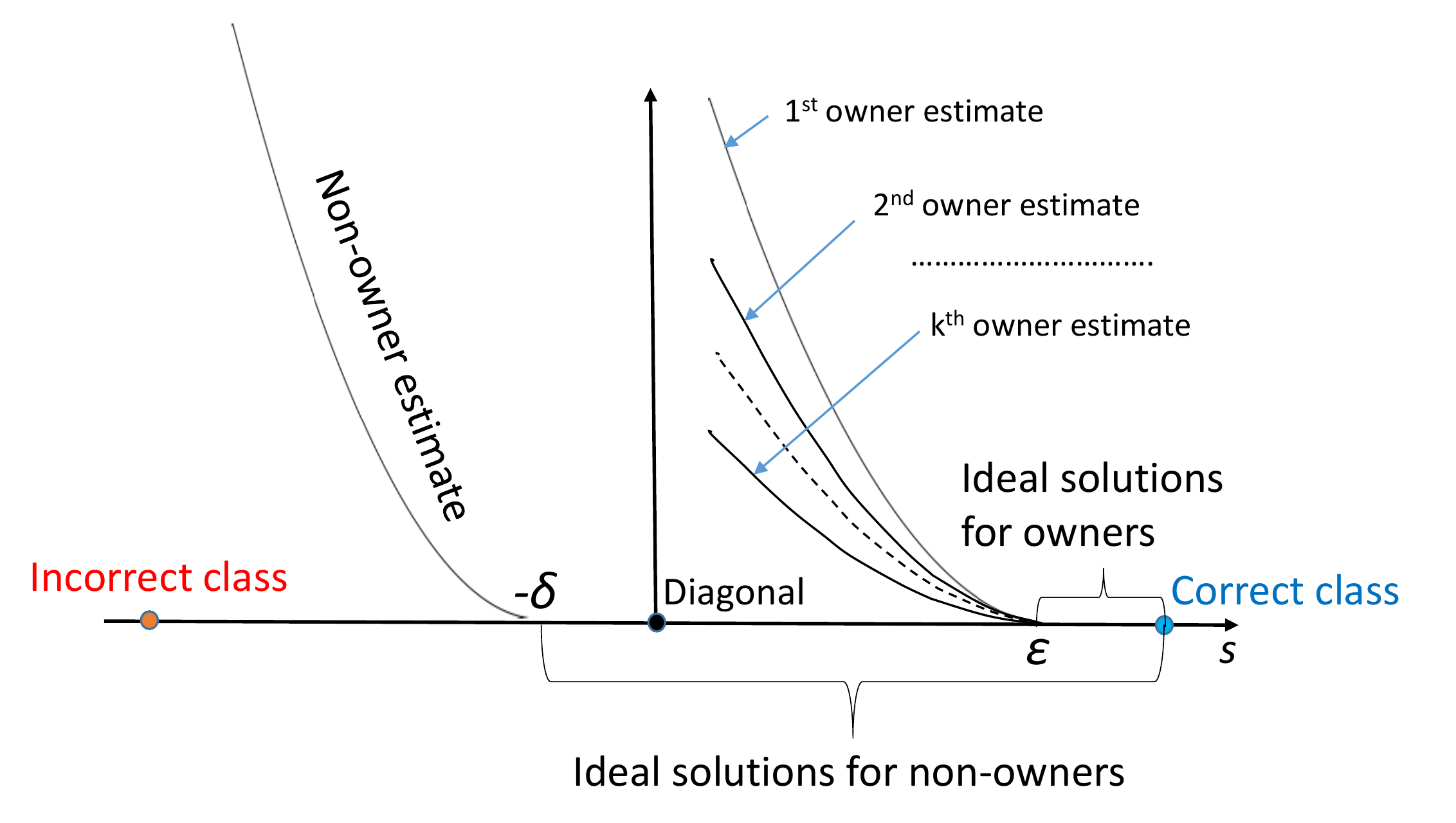}
\caption {Soft margin error function for owners and non-owners ( $k$  owners).}
\label{Fig:kWinnersTA}
\end{figure}

\section{Self-learning Stage: Communities and Recommender Systems}

After the stage of supervising learning, community of local experts can start working with new, previously unlabeled data. For a new example, the owners will be identified and the task will be solved by the owners following decision  {from} Figs.~\ref{Fig:WinnerTA}, \ref{Fig:Voting} or similar rules with distribution of responsibility between the most self-confident experts. After   such labeling steps the learning cycles should follow with improvement of experts' performance (they should give the correct self-confident answers to the samples they own, and do not make large mistakes for all other examples).

This regular alternation, solving new tasks -- learning -- solving new task -- ..., provides
adaptation to the graduate change in reality and assimilation of growing data. It is not compulsory that all local experts are answering the same tasks. A sort of soft {\em biclustering systems} of experts and problems should be implemented to link a problem to potential experts and an expert to tasks it can own. Selection of experts should be done with some excess  to guarantee sufficient number of selected skilled experts for correct solution. Originally \cite{GilGorbanMir1991}, a version of neural network associative memory was proposed to calculate the relative  weight  of an expert for solution of a problem (we can call it ``affinity of an  expert to a problem''). A well-developed technology of recommender systems \cite{recommender2015} includes many solutions potentially usable for recommendations of local experts to problems and problems to local experts.
Implementation of a recommender system for  the assignment of local experts to solve problems transforms the community of agents into hierarchical ``social network'' with various nodes and groups.

\section{Correctors, Knowledge Transfer, and Interiorisation}

Objectives of the  community self-learning are:
\begin{itemize}
\item{Assimilation of incrementally  growing data;}
\item{Adaptation to graduate change in reality;}
\item{Non-iterative knowledge transfer from the locally best experts to other agents;}
\end{itemize}

In the community self-learning process for each sample the locally best experts (owners) find the label. After the labeling, the skills should be improved. The supervised learning of large multiagent system requires large resources. It should no destroy the previous skills and, therefore, the large labeled data base of previous tasks should be used. It can require large memory and many iterations, which involve all the local experts. It is desirable to correct the errors (or increase the level of confidence, if it is too low) without destroying of previously learned skills. It is also very desirable to avoid usage of large database and long iterative process.

Communities of AI systems in real world will work on the basis of heterogeneous networks of computational devices and in heterogeneous infrastructure.  Real-time correction of  mistakes in such heterogeneous systems by re-training is not always viable due to the resources involved.
We can, therefore, formulate the technical requirements for the correction procedures \cite{GorTyukPhil2018}. {\em Corrector} should:
\begin{itemize}
\item be simple;
\item not destroy the existing skills of the AI systems;
\item allow fast non-iterative learning;
\item allow correction of new mistakes without destroying of previous corrections.
\end{itemize}
Surprisingly, all these requirements can be met  in sufficiently high dimensions. For this purpose, we propose to employ the concept of corrector of legacy AI systems, developed recently \cite{GorTyuRom2016,GorbanRomBurtTyu2016} on the basis of stochastic separation theorems \cite{GorbTyu2017}. For high-dimensional distributions in $n$-dimensional space every point from a large finite set can be separated from all other points by a simple linear discriminant. The size of this finite set can grow exponentially with $n$. For example, for the equidistribution in an $100$-dimensional ball,  with probability $>0.99$ every point in $2.7\cdot 10^6$ independently chosen random points  is linearly separable from the set of all other points.

The idea of a corrector is simple. It corrects an error of a single local expert. Separate  the sample with error from all other samples by a simple discriminant. This discriminant splits the space of samples into two subsets:  the sample with errors belongs to one of them, and all other samples belong to ``another half''. Modify the decision rule for the set, which includes the erroneous sample. This is corrector of a legacy AI system (Fig.~\ref{Fig:Corrector}). Inputs for this corrector may include input signals, and any internal or output signal of the AI system.

\begin{figure}
\centering
\includegraphics[width=0.95\columnwidth]{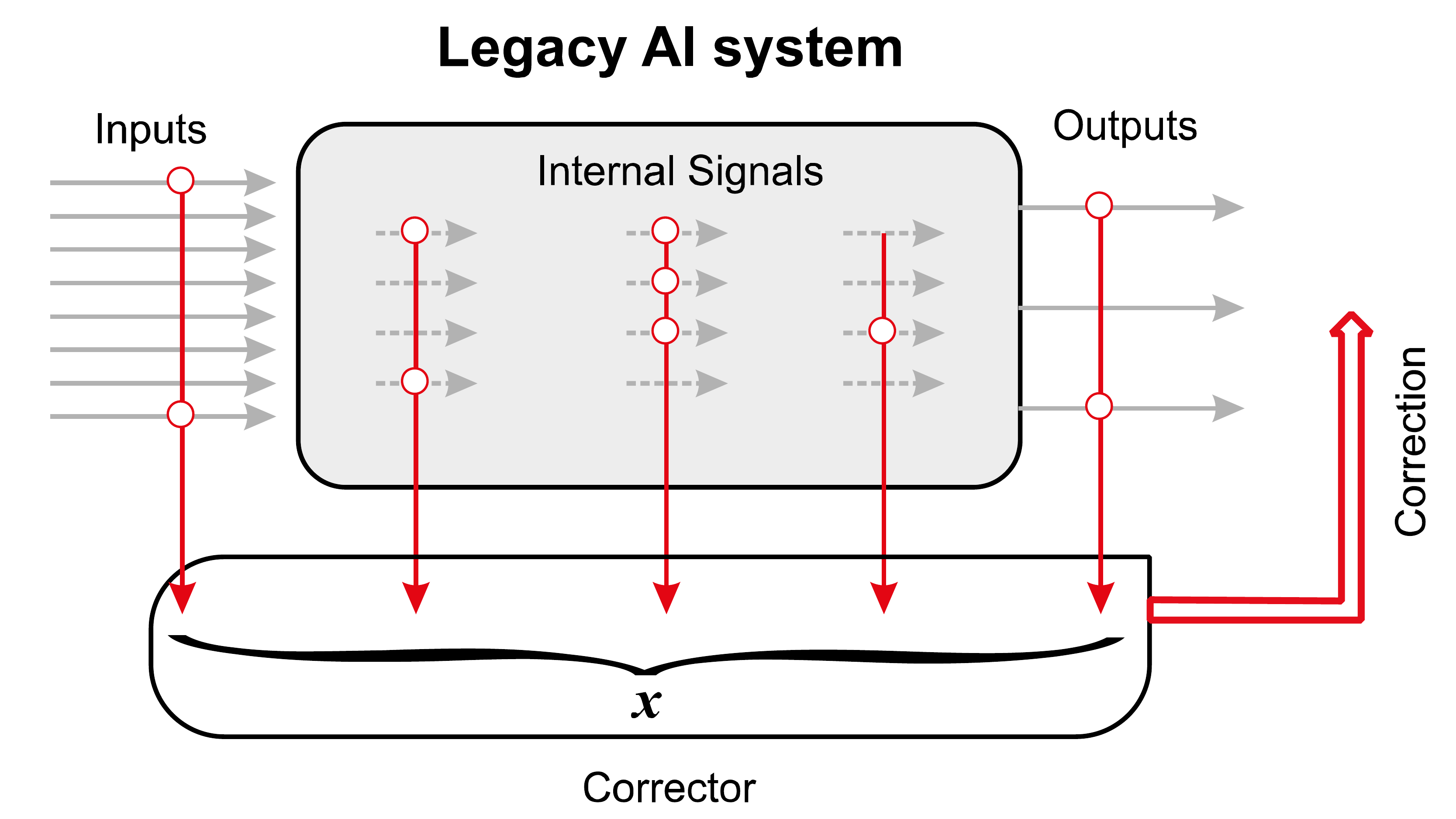}
\caption {Corrector of AI errors.}
\label{Fig:Corrector}
\end{figure}

One corrector can correct several errors (it is useful to cluster them before corrections). For correction of many errors, cascades of correctors are employed \cite{GorTyukPhil2018}: the AI system with the first corrector is a new legacy AI system and can be corrected further, as presented in Fig.~\ref{Fig:AIcorrectorsCascade}. In this diagram, the original legacy AI system (shown as Legacy AI System 1) is supplied with a corrector altering its responses. The combined new AI system can in turn be augmented by another corrector, leading to a cascade of AI correctors.

\begin{figure}
\centering
\includegraphics[width=0.95\columnwidth]{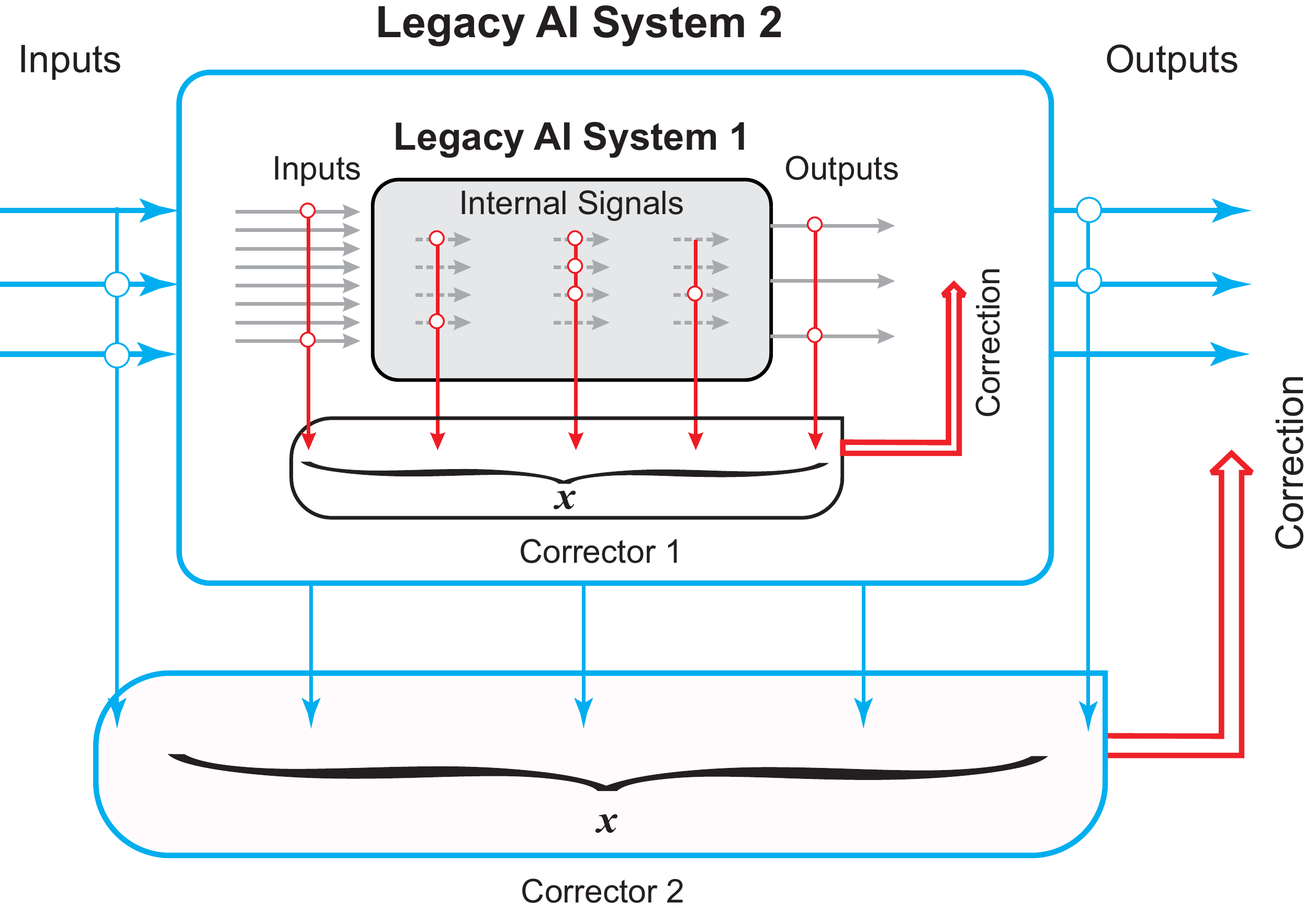}
\caption{Cascade of AI correctors}\label{Fig:AIcorrectorsCascade}
\end{figure}

Fast {\em  knowledge transfer} between AI systems can be organised using correctors \cite{Tyukin2017a}. The ``teacher'' AI  labels the samples, and a ``student'' AI also attempts  to label them. If their decisions coincide (with the desired level of confidence) then  nothing happens. If they do not coincide (or the level of confidence of a student is too low) then a corrector is created for the student. From the technological point of view it is more efficient to collect samples with student's errors, then cluster these samples and create correctors for the clusters, not for the individual mistakes. Moreover, new real-world samples are not compulsory needed in the knowledge transfer process. Just a large set of randomly generated (simulated) samples labeled by the teacher AI and the student AI can be used for correction of the student AI with skill transfer from the teacher AI.

\begin{figure*}
\centering
\includegraphics[width=0.8\textwidth]{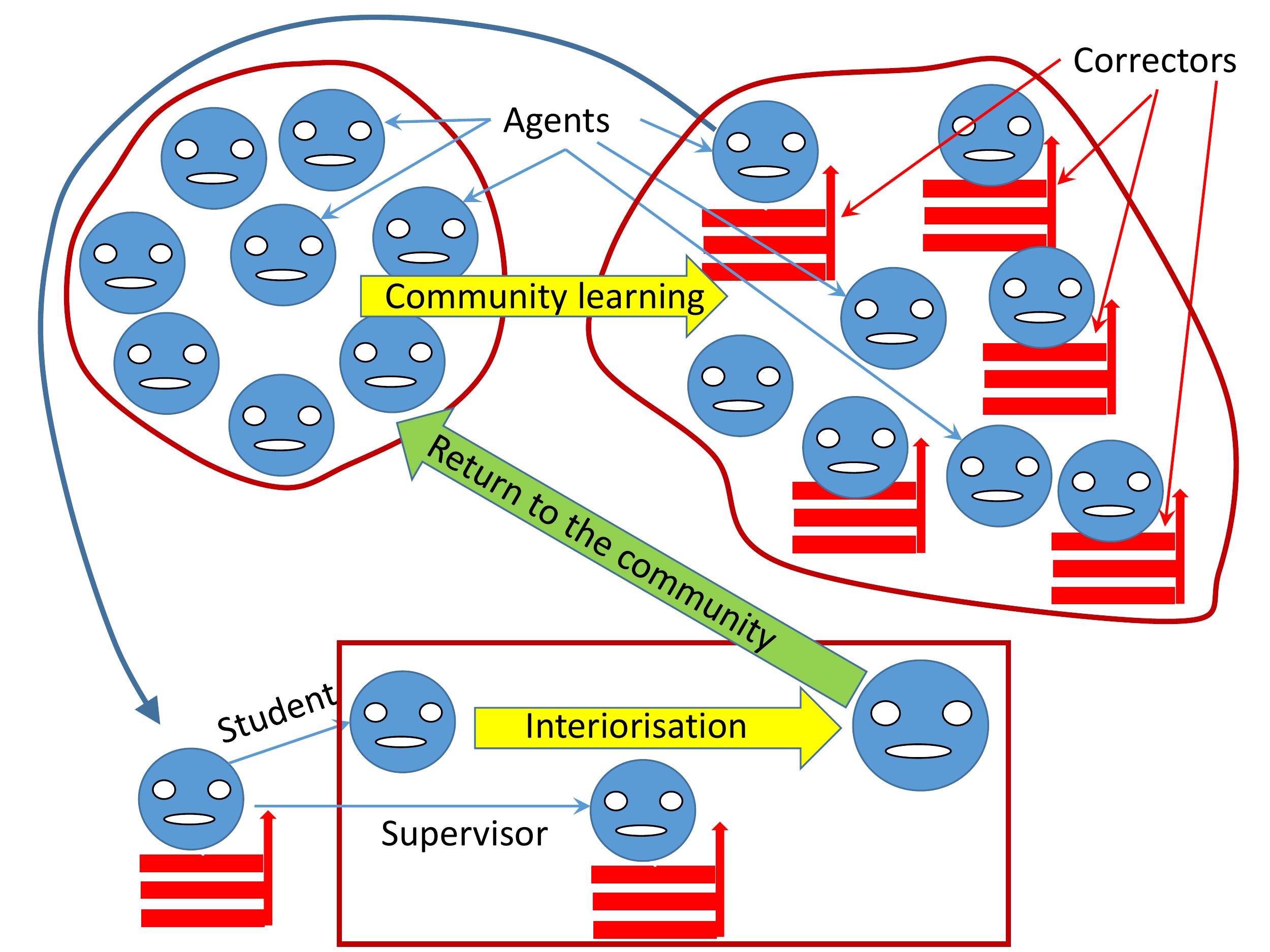}
\caption {Community learning, self-learning, and interiorisation of knowledge.}
\label{Fig:CommunityLearning}
\end{figure*}

 Correctors assimilate new knowledge in the course of the  community self-learning process (Fig.~\ref{Fig:CommunityLearning}). After collection of a sufficiently large cascade of correctors, a local expert needs to assimilate this knowledge in its internal structure. The main reason for such {\em interiorisation} is restoring of the regular essentially high-dimensional structure of the distribution of preprocessed  samples with preservation of skills. This process can be iterative but it is much simpler that the initial supervising learning. The local expert with the cascade of correctors becomes the teacher AI, and the same expert without correctors becomes the student AI (see Fig.~\ref{Fig:CommunityLearning}). Available real dataset can be supplemented by the randomly simulated samples and, after iterative learning the skills from the teacher are transferred to the student (if the capacity of the student is sufficient). The student with updated skills returns to the community of local experts.

Two important subsystems are not present in Fig.~\ref{Fig:CommunityLearning}): the manager -- recommender and the ultimate auditor. The {\em manager -- recommender} distributes tasks to local experts and local experts to tasks. It takes decisions on the basis of the previous experience of problem solving and assigns experts to problems with an adequate surplus, for reliability, and with some stochastisation, for the training of various experts and for the extension of experts' pool.

In practice, the self-learning and self-labeling of  samples performed by the selected local experts is supplemented by the labeling of  samples and critics of decisions by an {\em ultimate auditor}. First of all, this auditor is the real practice  itself: the real consequences of the decisions return to the systems. Secondly, the ultimate audit may include inspection by  a qualified human or by a special AI audit system with additional skills.

\section{Mathematical foundations of non-destructive  AI correction}

\subsection{General stochastic separation theorem}

 B{\'a}r{\'a}ny and  Zolt{\'a}n \cite{convhull} studied properties of high-dimensional polytopes deriving from uniform distribution in the  {$n$-dimensional} unit ball. They found that in the  envelope of $M$ random  {points} {\em  all } of the points are on the boundary of their convex hull and none belong to the interior (with probability greater than $1-c^2$, provided that $M \leq c 2^{n/2}$, where $c>0$ in an arbitrary constant). They also show that the bound on $M$ is nearly tight, up to polynomial factor in $n$. Donoho and Tanner \cite{DonohoTanner2009} derived a similar result for i.i.d.  points from the Gaussian distribution.
They also mentioned that in applications  it often seems that Gaussianity is not required and stated the problem of characterisation of   ensembles leading to the same qualitative effects (`phase transitions'), which are found  for Gaussian polytopes.

Recently, we noticed that these results could be proven for many other distributions, indeed, and one more important (and surprising) property is also typical: {\em  every point in  this $M$-point random set can be separated from  all other points of this set by the simplest linear Fisher discriminant} \cite{GorTyuRom2016,GorbTyu2017}. This observation allowed us to create the corrector technology for legacy AI systems \cite{GorbanRomBurtTyu2016}. We used the `thin shell' measure concentration inequalities to prove these results \cite{GorbTyu2017,GorTyukPhil2018}. Separation by linear Fisher's discriminant is practically most important {\em Surprise 4} in addition to three surprises mentioned in \cite{DonohoTanner2009}.

The standard approach  {assumes that} the random set consists of independent identically distributed (i.i.d.) random vectors. The new stochastic separation theorem presented below does not assume that the points are identically distributed. It can be very important: in the real practice the new datapoints are not compulsory taken from the same distribution that the previous points. In that sense the typical situation with the real data flow is far from an i.i.d. sample (we are grateful to G. Hinton for this important remark). This new theorem gives also an answer to the {\em open problem} from \cite{DonohoTanner2009}: it gives the general characterisation of the wide class of distributions with stochastic separation theorems (the SmAC condition below). Roughly speaking, this class consists of distributions without sharp peaks in sets with exponentially small volume (the precise formulation is below).
We call this property ``Smeared Absolute Continuity'' (or SmAC for short)
with respect to the Lebesgue measure: the absolute continuity means that the sets of zero measure have zero probability, and the SmAC condition  below requires that the sets with exponentially small volume should not have high probability.
%
%
Below ${\mathbb{B}}_n$ is a unit ball in ${\mathbb R}^n$ and $V_n$ denotes the $n$-dimensional Lebesgue measure.

Consider a  \emph{family} of distributions, one for each pair of positive integers $M$ and $n$. The general SmAC condition is
\begin{definition}
The joint distribution of $\boldsymbol{x}_1, \boldsymbol{ x}_2, \dots, \boldsymbol{ x}_M$ has SmAC property if there are exist constants $A>0$, {$B\in(0,1)$}, and $C>0$, such that for every positive integer $n$, any convex set $S \in {\mathbb R}^n$ such that
$$
\frac{V_n(S)}{V_n({\mathbb{B}}_n)} \leq A^n,
$$
any index $i\in\{1,2,\dots,M\}$, and any points $\boldsymbol{ y}_1, \dots, \boldsymbol{ y}_{i-1}, \boldsymbol{ y}_{i+1}, \dots, \boldsymbol{ y}_M$ in ${\mathbb R}^n$,
we have
\begin{equation}\label{eq:condstar}
{\mathbb P}(\boldsymbol{ x}_i \in {\mathbb{B}}_n \setminus S\, | \, \boldsymbol{ x}_j=\boldsymbol{ y}_j, \forall j \neq i) \geq 1-CB^n.
\end{equation}
\end{definition}

 We remark that
\begin{itemize}
\item We do not require for SmAC condition to hold for \emph{all} $A<1$, just for \emph{some} $A>0$. However, constants $A$, $B$, and $C$ should be independent from $M$ and $n$.
\item  We do not require that $\boldsymbol{ x}_i$ are independent. If they are, \eqref{eq:condstar} simplifies to
$$
{\mathbb P}(\boldsymbol{ x}_i \in {\mathbb{B}}_n \setminus S) \geq 1-CB^n.
$$
\item  We do not require that $\boldsymbol{ x}_i$ are identically distributed.
\item  The unit ball ${\mathbb{B}}_n$ in  SmAC condition can be replaced by an arbitrary ball, due to rescaling.
\item  We do not require the distribution to have a bounded support - points $\boldsymbol{ x}_i$ are allowed to be outside the ball, but with exponentially small probability.
\end{itemize}

The following proposition establishes a sufficient condition for SmAC condition to hold.

\begin{proposition}
Assume that $\boldsymbol{ x}_1, \boldsymbol{ x}_2, \dots, \boldsymbol{ x}_M$ are continuously distributed in ${\mathbb{B}}_n$ with conditional density satisfying
\begin{equation}\label{eq:condden}
\rho_n(\boldsymbol{ x}_i \,|\,\boldsymbol{ x}_j=\boldsymbol{ y}_j, \forall j \neq i) \leq \frac{C}{r^n V_n({\mathbb{B}}_n)}
\end{equation}
for any $n$, any index $i\in\{1,2,\dots,M\}$, and any points $\boldsymbol{ y}_1, \dots, \boldsymbol{ y}_{i-1}, \boldsymbol{ y}_{i+1}, \dots, \boldsymbol{ y}_M$ in ${\mathbb R}^n$, where $C>0$ and $r>0$ are some constants. Then SmAC condition holds with the same $C$, any $B \in (0,1)$, and $A=Br$.
\end{proposition}
\begin{proof}

\begin{equation*}
\begin{split}
{\mathbb P}(\boldsymbol{ x}_i \in S\, |& \, \boldsymbol{ x}_j=\boldsymbol{ y}_j, \forall j \neq i) = \int\limits_S \rho_n(\boldsymbol{ x}_i \,|\,\boldsymbol{ x}_j=\boldsymbol{ y}_j, \forall j \neq i) dV \\& \leq \int\limits_S \frac{C}{r^n V_n({\mathbb{B}}_n)} dV = V_n(S)\frac{C}{r^n V_n({\mathbb{B}}_n)} \\ &\leq A^n V_n({\mathbb{B}}_n)\frac{C}{r^n V_n({\mathbb{B}}_n)} = CB^n.
\end{split}
\end{equation*}

\end{proof}

If $\boldsymbol{ x}_1, \boldsymbol{ x}_2, \dots, \boldsymbol{ x}_M$ are independent with $\boldsymbol{ x}_i$ having density $\rho_{i,n}:{\mathbb{B}}_n \to [0,\infty)$,  \eqref{eq:condden} simplifies to
\begin{equation}\label{eq:indden}
\rho_{i,n}(\boldsymbol{ x}) \leq \frac{C}{r^n V_n({\mathbb{B}}_n)}, \quad \forall n, \, \forall i, \, \forall \boldsymbol{x} \in {\mathbb{B}}_n,
\end{equation}
where $C>0$ and $r>0$ are some constants.

With $r=1$, \eqref{eq:indden} implies that SmAC condition holds for probability distributions whose density is bounded by a constant times density $\rho_n^{uni}:=\frac{1}{V_n({\mathbb{B}}_n)}$ of uniform distribution in the unit ball. With arbitrary $r>0$, \eqref{eq:indden} implies that SmAC condition holds whenever ration $\rho_{i,n}/\rho_n^{uni}$ grows at most exponentially in $n$. This condition is general enough to hold for many distributions of practical interest.

\begin{example}{(Unit ball)}\label{ex:ball}
If $\boldsymbol{ x}_1, \boldsymbol{ x}_2, \dots, \boldsymbol{ x}_M$ are i.i.d random points from the equidistribution in the unit ball, then \eqref{eq:indden} holds with $C=r=1$.
\end{example}
\begin{example}{({Randomly perturbed} data)}\label{ex:noisy}
Fix parameter $\epsilon \in (0,1)$ ({random perturbation} parameter).
Let $\boldsymbol{ y}_1, \boldsymbol{ y}_2, \dots, \boldsymbol{ y}_M$ be the set of $M$ arbitrary (non-random) points inside the ball with radius $1-\epsilon$ in ${\mathbb R}^n$. They might be clustered in arbitrary way, all belong to a subspace of very low dimension, etc. Let
$\boldsymbol{ x}_i, i=1,2,\dots,M$ be a point, selected uniformly at random from a ball with center $\boldsymbol{ y}_i$ and radius $\epsilon$. We think about $\boldsymbol{ x}_i$ as ``perturbed'' version of $\boldsymbol{ y}_i$.
In this model, \eqref{eq:indden} holds with $C=1$, $r=\epsilon$.
\end{example}

\begin{example}{(Uniform distribution in a cube)}\label{ex:cube}
Let $\boldsymbol{ x}_1, \boldsymbol{ x}_2, \dots, \boldsymbol{ x}_M$ be i.i.d random points from the equidistribution in the unit cube. Without loss of generality, we can scale the cube to have side length $s=\sqrt{4/n}$. Then \eqref{eq:indden} holds with $r<\sqrt{\frac{2}{\pi e}}$. 
\end{example}
\begin{remark}
In this case,
\begin{equation*}
\begin{split}
&V_n({\mathbb{B}}_n)\rho_{i,n}(\boldsymbol{ x}) = \frac{V_n({\mathbb{B}}_n)}{(\sqrt{4/n})^n} = \frac{\pi^{n/2}/\Gamma(n/2+1)}{(4/n)^{n/2}}\\& < \frac{(\pi/4)^{n/2}n^{n/2}}{\Gamma(n/2)} \approx \frac{(\pi/4)^{n/2}n^{n/2}}{\sqrt{4\pi/n}(n/2e)^{n/2}} \leq \frac{1}{2\sqrt{\pi}}\left(\sqrt{\frac{\pi e}{2}}\right)^n,
\end{split}
\end{equation*}
where $\approx$ means Stirling's approximation for gamma function $\Gamma$.
\end{remark}

\begin{example}{(Product distribution in unit cube)}\label{ex:product}
Let $\boldsymbol{ x}_1, \boldsymbol{ x}_2, \dots, \boldsymbol{ x}_M$ be independent random points from the product distribution in the unit cube, with component $j$ of point $\boldsymbol{ x}_i$ having a continuous distribution with density $\rho_{i,j}$. Assume that all $\rho_{i,j}$ are bounded from above by some absolute constant $K$. Then \eqref{eq:indden} holds with $r<\frac{1}{K}\sqrt{\frac{2}{\pi e}}$ (after appropriate scaling of the cube).
\end{example}

A finite set $F \subset {\mathbb R}^n$ is called \emph{linearly separable} if the following equivalent conditions hold.
\begin{itemize}
\item  For each $\boldsymbol{ x} \in F$ there exists a linear functional $l$ such that $l(\boldsymbol{ x}) > l(\boldsymbol{ y})$ for all $\boldsymbol{y} \in F$, $\boldsymbol{ y}
\neq \boldsymbol{x}$;
\item  Each $\boldsymbol{ x} \in F$ is an extreme point (vertex) of convex hull of $F$.
\end{itemize}

Below we prove the separation theorem for distributions satisfying SmAC condition. The proof is based on the following result, see \cite{volume}

\begin{proposition}\label{prop:volest}
Let
$$
V(n,M)=\frac{1}{V_n({\mathbb{B}}_n)}\max\limits_{\boldsymbol{ x}_1, \dots, \boldsymbol{x}_M \in {\mathbb{B}}_n} V_n(\text{conv}\{\boldsymbol{x}_1, \dots, \boldsymbol{ x}_M\}),
$$
where $\text{conv}$ denotes the convex hull.
Then
$$
V(n,c^n)^{1/n} < (2e \log c)^{1/2}(1+o(1)), \quad 1<c<1.05.
$$
\end{proposition}
Proposition \ref{prop:volest} implies that for every $c \in (1, 1.05)$, there exists a constant $N(c)$, such that
\begin{equation}\label{eq:volest}
V(n,c^n) < ( {3} \sqrt{\log c})^{n}, \quad n>N(c).
\end{equation}

\begin{theorem}\label{th:separation}
Let $\{\boldsymbol{x}_1, \ldots , \boldsymbol{x}_M\}$ be a set of  random points in ${\mathbb R}^n$ from distribution satisfying SmAC condition. Then $\{\boldsymbol{x}_1, \ldots , \boldsymbol{x}_M\}$ is linearly separable with probability greater than $1-\delta$, $\delta>0$, provided that
$$
M \leq a b^n,
$$
where $$b=\min\{1.05, 1/B, \exp((A/3)^2)\},\;a=\min\{1, \delta/2C,b^{-N(b)}\}.$$
\end{theorem}
\begin{proof}
If $n<N(b)$, then $M \leq a b^n \leq b^{-N(b)} b^n < 1$, a contradiction. Let $n \geq N(b)$, and let $F=\{\boldsymbol{ x}_1, \boldsymbol{x}_2, \dots, \boldsymbol{ x}_M\}$. Then
\begin{equation*}
\begin{split}
{\mathbb P}(F \subset {\mathbb{B}}_n) &\geq 1 - \sum_{i=1}^M {\mathbb P}(\boldsymbol{x}_i \not \in {\mathbb{B}}_n)\\ &\geq 1 - \sum_{i=1}^M CB^n = 1-MCB^n,
\end{split}
\end{equation*}
where the second inequality follows from \eqref{eq:condstar}.
Next,
\begin{equation*}
\begin{split}
{\mathbb P}(F \, &\text{is linearly separable}\,|\,F \subset {\mathbb{B}}_n)\\ &\geq 1-\sum_{i=1}^M {\mathbb P}(\boldsymbol{ x}_i \in \text{conv}(F\setminus \{\boldsymbol{ x}_i\})\,|\,F \subset {\mathbb{B}}_n).
\end{split}
\end{equation*}
%
For set $S=\text{conv}(F\setminus \{\boldsymbol{ x}_i\})$
\begin{equation*}
\begin{split}
\frac{V_n(S)}{V_n({\mathbb{B}}_n)}& \leq  V(n,M-1) \leq V(n,b^n)  \\& < \left(3\sqrt{\log (b)}\right)^n  \leq A^n,
\end{split}
\end{equation*}
where we have used \eqref{eq:volest} and inequalities $a \leq 1$,  $b \leq \exp((A/3)^2)$. Then SmAC condition implies that
\begin{equation*}
\begin{split}
{\mathbb P}(\boldsymbol{ x}_i &\in \text{conv}(F\setminus \{\boldsymbol{ x}_i\})\,|\,F \subset {\mathbb{B}}_n)\\
& = {\mathbb P}(\boldsymbol{ x}_i \in S\,|\,F \subset {\mathbb{B}}_n) \leq CB^n.
\end{split}
\end{equation*}
Hence,
$$
{\mathbb P}(F \, \text{is linearly separable}\,|\,F \subset {\mathbb{B}}_n) \geq 1 - MCB^n,
$$
and
\begin{equation*}
\begin{split}
{\mathbb P}(F \,& \text{is linearly separable}) \geq (1 - MCB^n)^2\\ & \geq 1-2MCB^n \geq 1-2ab^{n}CB^n \geq 1-\delta,
\end{split}
\end{equation*}
where the last inequality follows from $a \leq \delta/2C$,  $b\leq 1/B$.
\end{proof}

\subsection{Stochastic separation by Fisher's linear discriminant}

According to the  general stochastic separation theorems there {\em exist}  linear functionals, which  separate points in a random set (with high probability and under some conditions). Such a linear functional can be found by various iterative methods, from the Rosenblatt perceptron learning rule to support vector machines.  This existence is nice but for applications we need the non-iterative learning. It would be very desirable to have an explicit expression for separating functionals.

 There exists a general scheme for creation of linear discriminants {\cite{Tyukin2017a,GorTyukPhil2018}}. For separation of single  points from a data cloud it is necessary:
 \begin{enumerate}
 \item Centralise the cloud (subtract the mean point from all data vectors).
 \item Escape strong multicollinearity, for example, by principal component analysis and deleting minor components, which correspond to the small eigenvalues of empiric covariance matrix.
\item Perform whitening (or spheric transformation), that is a linear transformation, after that the covariance matrix becomes a unit matrix. In principal components whitening is simply the normalisation of coordinates to unit variance.
\item The linear inequality for separation of a point $\boldsymbol{ x}$ from the cloud $Y $ in new coordinates is
\begin{equation}\label{discriminant}
(\boldsymbol{x},\boldsymbol{y})\leq \alpha (\boldsymbol{x},\boldsymbol{x}), \mbox{  for all  } \boldsymbol{y}\in Y.
\end{equation}
 where $\alpha\in (0,1)$ is a threshold, and $(\bullet,\bullet)$ is the standard Euclidean inner product in new coordinates.
\end{enumerate}

In real applied problems, it could be difficult to perform the precise whitening but a rough approximation to this transformation could also create useful discriminants (\ref{discriminant}). We will call `Fisher's discriminants' all the discriminants created non-iteratively by inner products (\ref{discriminant}), with some extension of  meaning.

Formally, we say that finite set $F \subset {\mathbb R}^n$ is \emph{Fisher-separable} if
\begin{equation}\label{eq:Fisher}
(\boldsymbol{ x},\boldsymbol{ x}) > (\boldsymbol{ x},\boldsymbol{ y}), 
\end{equation}
holds for all $\boldsymbol{ x}, \boldsymbol{ y} \in F$ such that $\boldsymbol{ x}\neq  \boldsymbol{ y}$.

Two following theorems demonstrate that Fisher's  discriminants are powerful in high dimensions.

\begin{theorem}[Equidistribution in $\mathbb{B}_n$ \cite{GorTyuRom2016,GorbTyu2017}]\label{ball1point}
Let $\{\boldsymbol{x}_1, \ldots , \boldsymbol{x}_M\}$ be a set of $M$  i.i.d. random points  from the equidustribution in the unit ball $\mathbb{B}_n$. Let $0<r<1$, and $\rho=\sqrt{1-r^2}$. Then
\begin{equation}
\begin{split}\label{Eq:ball1}
&\mathbf{P}\left(\|\boldsymbol{x}_M\|>r \mbox{ and }
\left(\boldsymbol{x}_i,\frac{\boldsymbol{x}_M}{\| \boldsymbol{x}_M\| } \right)
<r \mbox{ for all } i\neq M \right) \\& \geq 1-r^n-0.5(M-1) \rho^{n};
\end{split}
\end{equation}
\begin{equation}
\begin{split}\label{Eq:ballM}
&\mathbf{P}\left(\|\boldsymbol{x}_j\|>r  \mbox{ and } \left(\boldsymbol{x}_i,\frac{\boldsymbol{x}_j}{\| \boldsymbol{x}_j\|}\right)<r \mbox{ for all } i,j, \, i\neq j\right) \\& \geq  1-Mr^n-0.5M(M-1)\rho^{n};
\end{split}
\end{equation}
\begin{equation}
\begin{split}\label{Eq:ballMangle}
&\mathbf{P}\left(\|\boldsymbol{x}_j\|>r  \mbox{ and } \left(\frac{\boldsymbol{x}_i}{\| \boldsymbol{x}_i\|},\frac{\boldsymbol{x}_j}{\| \boldsymbol{x}_j\|}\right)<r \mbox{ for all } i,j, \,i\neq j\right)\\&  \geq  1-Mr^n-M(M-1)\rho^{n}.
\end{split}
\end{equation}
\end{theorem}

According to Theorem \ref{ball1point}, the probability that a single element $\bfx_M$ from the sample $\mathcal{S}=\{\bfx_1,\dots,\bfx_{M}\}$ is linearly separated from the set $\mathcal{S}\setminus \{\bfx_M\}$ by the hyperplane $l(x)=r$ is at least
\[
1-r^n-0.5(M-1)\left(1-r^2\right)^{\frac{n}{2}}.
\]
This probability estimate depends on both $M=|\mathcal{S}|$ and dimensionality $n$.  An interesting consequence of the theorem is that if one picks a probability value, say $1-\vartheta$, then the maximal possible values of $M$ for which the set $\mathcal{S}$ remains linearly separable with  probability that is no less than $1-\vartheta$ grows at least exponentially with $n$. In particular, the following holds

 \begin{corollary}\label{cor:exponential}
 Let $\{\boldsymbol{x}_1, \ldots , \boldsymbol{x}_M\}$ be a set of $M$ i.i.d. random points  from the equidustribution in the unit ball $\mathbb{B}_n$. Let $0<r,\vartheta<1$, and $\rho=\sqrt{1-r^2}$. If
\begin{equation}\label{EstimateMball}
M<2({\vartheta-r^n})/{\rho^{n}},
 \end{equation}
 then
 $
 \mathbf{P}((\boldsymbol{x}_i,\boldsymbol{x}_M{)}<r\|\boldsymbol{x}_M\| \mbox{ for all } i=1,\ldots, M-1)>1-\vartheta.
$
 If
 \begin{equation}\label{EstimateM2ball}
M<({r}/{\rho})^n\left(-1+\sqrt{1+{2 \vartheta \rho^n}/{r^{2n}}}\right),
 \end{equation}
  then $\mathbf{P}((\boldsymbol{x}_i,\boldsymbol{x}_j)<r\|\boldsymbol{x}_i\| \mbox{ for all } i,j=1,\ldots, M, \, i\neq j)\geq 1-\vartheta.$

  In particular, if inequality (\ref{EstimateM2ball}) holds then the set $\{\boldsymbol{x}_1, \ldots , \boldsymbol{x}_M\}$ is  Fisher-separable  with probability $p>1-\vartheta$.
 \end{corollary}

{Note that (\ref{Eq:ballMangle}) implies that elements of the set $\{\boldsymbol{x}_1, \ldots , \boldsymbol{x}_M\}$ are pair-wise almost or $\varepsilon$-orthogonal, i.e. $|\cos(\bfx_i,\bfx_j)|\leq \varepsilon$ for all $i\neq j$, $1\leq i,j\leq M$,  with probability larger or equal than  $1-2Mr^n-2M(M-1)\rho^{n}$. Similar to Corollary \ref{cor:exponential}, one  can conclude that the cardinality $M$ of samples with such properties grows at least exponentially with $n$. Existence of the phenomenon has been demonstrated in \cite{Kurkova1993}. Theorem \ref{ball1point},  Eq. (\ref{Eq:ballMangle}), shows that the phenomenon is typical in some sense (cf. \cite{bases}, \cite{Kurkova:2017}).}

The linear separability property of finite but exponentially large samples of random i.i.d. elements is not restricted to equidistributions in $\mathbb{B}_n$. As has been noted in \cite{GorbanRomBurtTyu2016}, it holds for equidistributions in ellipsoids as well as for the Gaussian distributions. Moreover, it can be generalized to product distributions in a unit cube. Consider, e.g. {the case}  when coordinates of the vectors $\bfx=(X_1,\dots,X_n)$ in the set $\mathcal{S}$ are independent random variables $X_i$, $i=1,\dots,n$ with expectations $\overline{X}_i$ and variances $\sigma_i^2>\sigma_0^2>0$. Let $0\leq X_i\leq 1$ for all $i=1,\dots,n$. The following analogue of Theorem \ref{ball1point} can now be stated.

\begin{theorem}[Product distribution in a cube \cite{GorbTyu2017}]\label{cube} Let  $\{\boldsymbol{x}_1, \ldots , \boldsymbol{x}_M\}$ be i.i.d. random points from the product distribution in a unit cube. Let
\[
R_0^2=\sum_i \sigma_i^2\geq n\sigma_0^2.
\]
Assume that data are centralised and  $0< \delta <2/3$. Then
\begin{equation}\label{Eq:cube1}
\begin{split}
&\mathbf{P}\left(1-\delta  \leq \frac{\|\boldsymbol{x}_j\|^2}{R^2_0}\leq 1+\delta \mbox{ and }
\frac{(\boldsymbol{x}_i,\boldsymbol{x}_M)}{R_0\| \boldsymbol{x}_M\| }<\sqrt{1-\delta} \right. \\   &\quad \mbox{ for all } i,j, \, i\neq M {\bigg)} \geq 1- 2M\exp \left(-2\delta^2 R_0^4/n \right) \\ &\quad -(M-1)\exp \left(-2R_0^4(2-3 \delta)^2/n\right);
\end{split}
\end{equation}
\begin{equation}\label{Eq:cube2}
\begin{split}
&\mathbf{P}\left(1-\delta  \leq \frac{\|\boldsymbol{x}_j\|^2}{R^2_0}\leq 1+\delta \mbox{ and }
\frac{(\boldsymbol{x}_i,\boldsymbol{x}_j)}{R_0\| \boldsymbol{x}_j\| }<\sqrt{1-\delta} \right. \\ &\quad \mbox{ for all } i,j, \, i\neq j {\bigg)} \geq 1- 2M\exp \left(-2\delta^2 R_0^4/n \right) \\ & \quad -M(M-1)\exp \left(-2R_0^4(2-3 \delta)^2/n\right).
\end{split}
\end{equation}
\end{theorem}

In particular, under the conditions of Theorem \ref{cube}, set $\{\boldsymbol{x}_1, \ldots , \boldsymbol{x}_M\}$ is Fisher-separable  with probability $p>1-\vartheta$, provided that $M \leq ab^n$, where $a>0$ and $b>1$ are some constants depending only on $\vartheta$ and $\sigma_0$.

The proof  of Theorem \ref{cube} is based on concentration inequalities in product spaces \cite{Talagrand1995}. Numerous generalisations of Theorems
\ref{ball1point}, \ref{cube} are possible for different classes of distributions, for example, for weakly dependent variables, etc.

We can see from Theorem \ref{cube} that the discriminant (\ref{discriminant}) works without precise whitening. Just the absence of strong degeneration is required: the support of the distribution contains in the unit cube (that is bounding from above) and, at the same time, the variance of each coordinate is bounded from below by $\sigma_0>0$.

Linear separability, as an inherent property of data sets in high dimension, is not necessarily confined to cases whereby a linear functional separates a single element of a set from the rest.  Theorems \ref{ball1point}, \ref{cube} be generalized to account for $m$-tuples, $m>1$ too \cite{Tyukin2017a,GorTyukPhil2018}.

\begin{figure}
\centering
\includegraphics[width=0.5\columnwidth]{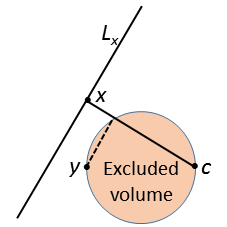}
\caption {Point $x$ should not belong to the filled ball (excluded volume) to be separable from $y$ by the linear discriminant (\ref{discriminant}). Here, $c$ is the centre of data cloud, and $L_x$ is the hyperplane such that $(x,z)=(x,x)$ for $z \in L_x$.}
\label{Fig:Excluded}
\end{figure}

Let us make several remarks for the general distributions. For each data point $y$ the probability that  a randomly chosen point $x$ is {\em not} separated from $y$ by the discriminant (\ref{eq:Fisher}) (i.e. that the inequality (\ref{eq:Fisher}) is false) is (Fig.~\ref{Fig:Excluded})
\begin{equation}\label{excluded}
p=p_y=\int_{\left\|z-\frac{y}{2}\right\|\leq \frac{\|y\|}{2}} \rho(z) dz ,
\end{equation}
where $\rho(z) dz$ is the probability measure.
We need to evaluate the probability of  finding a random point outside the union of $N$ such excluded volumes. For example,  for the equidistribution in a ball $\mathbb{B}_n$, $p_y <1/2^n$ for all $y$ from the ball. The probability to select a point  {inside} the union of $N$ `forbidden balls'  is less than $N/2^n$ for any position of $N$ points $y$ in $\mathbb{B}_n$. Inequalities (\ref{Eq:ball1}), (\ref{Eq:ballM}), and (\ref{Eq:ballMangle}) are closely related to  that fact.

Instead of equidistribution in a ball $\mathbb{B}_n$, we can take probability distributions  with bounded density $\rho$ in a ball  $\mathbb{B}_n$
\begin{equation}\label{bounded}
\rho(y)<\frac{C}{r^n V_n(\mathbb{B}_n)},
 \end{equation}
 where $C>0$ is an arbitrary constant, $V_n(\mathbb{B}_n)$ is the volume of the ball, and radius $r>1/2$. This inequality guarantees that the probability of each ball with radius less or equal than 1/2 exponentially decays for $n\to \infty$. It should be stressed that in asymptotic analysis for large $n$ the constant $C>0$ is arbitrary but  does not  {depend on $n$.} 

For the bounded distributions (\ref{bounded}) the separability by linear discriminants (\ref{discriminant}) is similar to separability for the equidistributions. The proof through the estimation of the probability to avoid the excluded volume is straightforward.

For the practical needs, the number  {of} points $N$ is large. We have to evaluate the sum $p_y$ for $N$ points $y$.  {For} most estimates, evaluation of the expectations and variances  of  $p_y$ (\ref{excluded}) will be sufficient, if points $y$ are independently chosen.

If the distribution is unknown and exists just in the form of the sample of empirical points, we can evaluate ${\bf E}(p)$ and $var( p)$  (\ref{excluded}) from the sample directly, without knowledge of theoretical probabilities.

Bound \eqref{bounded} is the special case of \eqref{eq:indden} with $r>1/2$, and it is  more restrictive: in Examples \ref{ex:noisy}, \ref{ex:cube}, and \ref{ex:product}, the distributions satisfy  \eqref{eq:indden} with some $r<1$, but fail to satisfy \eqref{bounded}  {if $r<1/2$}. Such distributions has the SmAC property and the corresponding set of points $\{\boldsymbol{x}_1, \ldots , \boldsymbol{x}_M\}$ is linearly separable by Theorem \ref{th:separation}, but different technique is needed to establish its Fisher-separability.  {One option is to} estimate the distribution of $p$  {in} (\ref{excluded}).
Another technique is based on concentration inequalities. For some distributions, one can prove that, with exponentially high probability, random point $\boldsymbol{x}$ satisfies
\begin{equation}\label{eq:thinshell}
r_1(n) \leq \|\boldsymbol{x}\| \leq r_2(n),
\end{equation}
where $\|\bullet\|$ denotes the Euclidean norm in ${\mathbb R}^n$, and $r_1(n)$ and $r_2(n)$ are some lower and upper bounds, depending on $n$. If $r_2(n)-r_1(n)$ is small comparing to $r_1(n)$, it means that the distribution is concentrated in a thin shell between the two spheres. If $\boldsymbol{x}$ and $\boldsymbol{y}$ satisfy \eqref{eq:thinshell}, inequality (\ref{eq:Fisher}) may fail only if $\boldsymbol{y}$ belongs to a ball with radius $R=\sqrt{r^2_2(n)-r^2_1(n)}$. If $R$ is much lower than $r_1(n)/2$, this method may provide much better probability estimate than \eqref{excluded}. This is how Theorem \ref{cube} was proved in \cite{GorbTyu2017}.

\section{Separation theorem for log-concave distributions}

\subsection{Log-concave distributions}

In \cite{GorbTyu2017} we proposed several possible generalisations of Theorems~\ref{ball1point}, \ref{cube}. One of them is the hypothesis that for the uniformly log-concave distributions the similar result can be formulated and proved. Below we demonstrate that this hypothesis is true, formulate and prove the stochastic separation theorems for several classes of log-concave distributions. Additionally, we prove the comparison (domination) Theorem~\ref{prop:domin} that allows to extend the proven theorems to wider classes of distributions. 

In this subsection, we introduce several classes of log-concave distributions and prove some useful properties of these distributions.

Let ${\cal P}=\{{\mathbb P}_n, \, n=1,2,\dots\}$ be a family of probability measures with densities $\rho_n:{\mathbb R}^n \to [0,\infty), \, n=1,2,\dots$. Below, $\boldsymbol{x}$ is a random variable (r.v) with density $\rho_n$, and ${\mathbb E}_n[f(\boldsymbol{x})]:=\int_{{\mathbb R}^n} f(z) \rho_n(z) dz$ is the expectation of $f(\boldsymbol{x})$. 

We say that density $\rho_n:{\mathbb R}^n \to [0,\infty)$ (and the corresponding probability measure ${\mathbb P}_n$): 
\begin{itemize}
\item is whitened, or \emph{isotropic}, if ${\mathbb E}_n[\boldsymbol{x}]=0$, and
\begin{equation}\label{eq:isot}
{\mathbb E}_n[(\boldsymbol{x},\theta)^2)]=1\quad\quad \forall \theta \in S^{n-1},
\end{equation} 
where $S^{n-1}$ is the unit sphere in ${\mathbb R}^n$, and $(\bullet,\bullet)$ is the standard Euclidean inner product in ${\mathbb R}^n$. The last condition is equivalent to the fact that the covariance matrix of the components of $\boldsymbol{x}$ is the identity matrix, see \cite{Lovasz}. 
\item is \emph{log-concave}, if set $D_n=\{z\in{\mathbb R}^n \,|\, \rho_n(z)>0\}$ is convex and $g(z)=-\log(\rho_n(z))$ is a convex function on $D_n$. 
\item  is \emph{strongly log-concave} (SLC), if $g(z)=-\log(\rho_n(z))$ is strongly convex, that is, there exists a constant $c>0$ such that
$$
\frac{g(u)+g(v)}{2} - g\left(\frac{u+v}{2}\right) \geq c ||u-v||^2, \quad\quad \forall u,v \in D_n.
$$
For example, density $\rho_G(z)=\frac{1}{\sqrt{(2\pi)^n}}\exp\left(-\frac{1}{2}||z||^2\right)$ of $n$-dimensional standard normal distribution is strongly log-concave with $c=\frac{1}{8}$.
\item has sub-Gaussian decay for the norm (SGDN), if there exists a constant $\epsilon>0$ such that
\begin{equation}\label{eq:sgdn}
{\mathbb E}_n[\exp\left(\epsilon ||\boldsymbol{x}||^2\right)] < +\infty.
\end{equation}
In particular, \eqref{eq:sgdn} holds for $\rho_G$ with any $\epsilon<\frac{1}{2}$. 
However, unlike SLC, \eqref{eq:sgdn} is an asymptotic property, and is not affected by local modifications of the underlying density. For example, density $\rho(z)= \frac{1}{C}\exp(-g(||z||)), \, z \in {\mathbb R}^n$, where $g(t)=\frac{1}{2}\max\{1,t^2\}, \, t \in {\mathbb R}$ and $C=\int_{{\mathbb R}^n}\exp(-g(||z||))dz$ has SGDN with any $\epsilon<\frac{1}{2}$, but it is not strongly log-concave. 
\item  has sub-Gaussian decay in every direction (SGDD), if there exists a constant $B>0$ such that inequality
$$
{\mathbb P}_n[(\boldsymbol{x},\theta) \geq t] \leq 2 \exp \left( -\frac{t}{B}\right)^2
$$ 
holds for every $\theta \in S^{n-1}$ and $t > 0$. 
\item  is $\psi_\alpha$ with constant $B_\alpha > 0$, $\alpha\in[1,2]$, if 
\begin{equation}\label{eq:psialpha}
\left({\mathbb E}_n|(\boldsymbol{x},\theta)|^p\right)^{1/p} \leq B_\alpha p^{1/\alpha}\left({\mathbb E}_n|(\boldsymbol{x},\theta)|^2\right)^{1/2}
\end{equation}
holds for every for every $\theta \in S^{n-1}$ and all $p \geq 2$.
\end{itemize}


\begin{proposition}\label{prop:inclusions}
Let $\rho_n:{\mathbb R}^n \to [0,\infty)$ be an isotropic log-concave density, and let $\alpha\in[1,2]$. The following implications hold.
\begin{equation*} 
\begin{split}
&\boxed{\rho_n \, \text{is SLC}}
\Rightarrow\boxed{\rho_n \, \text{has SGDN}}
\Rightarrow\boxed{\rho_n \, \text{has SGDD}}\Leftrightarrow \\
&\Leftrightarrow\boxed{\rho_n \, \text{is} \, \psi_2}
\Rightarrow\boxed{\rho_n \, \text{is} \, \psi_\alpha}
\Rightarrow\boxed{\rho_n \, \text{is} \, \psi_1}
\Leftrightarrow\boxed{\text{ALL}}\,\,,
\end{split}
\end{equation*} 
where the last $\Leftrightarrow$ means the class of isotropic log-concave densities which are $\psi_1$ actually coincides with the class of \emph{all} isotropic log-concave densities.
\end{proposition}
\begin{proof}
Proposition 3.1 in \cite{Bobkov} states that if there exists $c_1>0$ such that
$g(\boldsymbol{x})=-\log(\rho_n(\boldsymbol{x}))$ satisfies
\begin{equation}\label{eq:allts}
t g(u)+s g(v) - g\left(t u+s v\right) \geq \frac{c_1ts}{2} ||u-v||^2, \;  \forall u,v \in D_n.
\end{equation}
for all $t,s>0$ such that $t+s=1$, then inequality
\begin{equation}\label{eq:sobolev}
\begin{split}
&{\mathbb E}_n[f^2(\boldsymbol{x})\log f^2(\boldsymbol{x})] - {\mathbb E}_n[f^2(\boldsymbol{x})]{\mathbb E}_n[\log f^2(\boldsymbol{x})] \leq \\
&\;\; \leq \frac{2}{c_1} {\mathbb E}_n[||\nabla f(\boldsymbol{x})||^2]
\end{split}
\end{equation}
holds for every smooth function $f$ on ${\mathbb R}^n$. As remarked in \cite[p. 1035]{Bobkov}, ``it is actually enough that \eqref{eq:allts} holds for some $t, s > 0, t+s = 1$''. With $t=s=1/2$, this implies that \eqref{eq:sobolev} holds for every strongly log-concave distribution, with $c_1=8c$. By \cite[Theorem 3.1]{Bobkov2}, \eqref{eq:sobolev} holds for $\rho_n$ if and only if it has has sub-Gaussian decay for the norm, and the implication $\boxed{\rho_n \, \text{is SLC}}\Rightarrow\boxed{\rho_n \, \text{has SGDN}}$ follows. Also, by \cite[Theorem 1(i)]{Stavrakakis}, if \eqref{eq:sobolev} holds for $\rho_n$, then it is $\psi_2$ with constant $B_2=d/\sqrt{c}$, where $d$ is a universal constant, hence $\boxed{\rho_n \, \text{has SGDN}}\Rightarrow \boxed{\rho_n \, \text{is} \, \psi_2}\,$. The equivalence $\boxed{\rho_n \, \text{has SGDD}}\Leftrightarrow\boxed{\rho_n \, \text{is} \, \psi_2}$ follows from \eqref{eq:isot} and \cite[Lemma 2.2.4]{Brazitikos}. The implications $\boxed{\rho_n \, \text{is} \, \psi_2} \Rightarrow\boxed{\rho_n \, \text{is} \, \psi_\alpha} \Rightarrow\boxed{\rho_n \, \text{is} \, \psi_1}$ follow from \eqref{eq:psialpha}. Finally, \cite[Theorem 2.4.6]{Brazitikos} implies that every log-concave density $\rho_n$ is $\psi_1$ with some universal constant.
\end{proof}

\subsection{Fisher-separability for log-concave distributions}
 
Below we prove Fisher-separability for i.i.d samples from isotropic log-concave $\psi_\alpha$ distributions, using the the technique based on concentration inequalities.

\begin{theorem}\label{th:logconc}
Let $\alpha \in [1,2]$, and let 
${\cal P}=\{{\mathbb P}_n, \, n=1,2,\dots\}$ be a family of probability measures with densities $\rho_n:{\mathbb R}^n \to [0,\infty), \, n=1,2,\dots$,
which are $\psi_\alpha$ with constant $B_\alpha > 0$, independent from $n$.
Let $\{\boldsymbol{x}_1, \ldots , \boldsymbol{x}_M\}$ be a set of $M$ i.i.d. random points from $\rho_n$. Then there exist constants $a>0$ and $b>0$, which depends only on $\alpha$ and $B_\alpha$, such that, for any $i,j \in \{1,2,\dots,M\}$, inequality 
$$
(\boldsymbol{ x_i},\boldsymbol{ x_i}) > (\boldsymbol{ x_i},\boldsymbol{ x_j})
$$ 
holds with probability at least $1-a\exp(-b n^{\alpha/2})$. Hence, for any $\delta>0$, set $\{\boldsymbol{x}_1, \ldots , \boldsymbol{x}_M\}$ is Fisher-separable with probability greater than $1-\delta$, provided that
\begin{equation}\label{eq:Mbound}
M \leq \sqrt{\frac{2\delta}{a}}\exp\left(\frac{b}{2}n^{\alpha/2}\right).
\end{equation}
\end{theorem}
\begin{proof}
Let $\boldsymbol{x}$ and $\boldsymbol{y}$ be two points, selected independently at random from the distribution with density $\rho_n$.
\cite[Theorem 1.1]{Guedon}, (applied with $A=I_n$, where $I_n$ is $n \times n$ identity matrix) 
states that, for any $t\in(0,1)$,
\eqref{eq:thinshell} holds with $r_1(n)=(1-t)\sqrt{n}$, $r_2(n)=(1+t)\sqrt{n}$, and with probability at least $1-A\exp(-Bt^{2+\alpha}n^{\alpha/2})$, where $A,B>0$ are constants depending only on $\alpha$. If \eqref{eq:thinshell} holds for $\boldsymbol{x}$ and $\boldsymbol{y}$, inequality (\ref{eq:Fisher}) may fail only if $\boldsymbol{y}$ belongs to a ball with radius $R_n=\sqrt{r^2_2(n)-r^2_1(n)}=\sqrt{4tn}$.
Theorem 6.2 in \cite{Paouris}, applied with $A=I_n$, 
states that, for any $\epsilon\in(0,\epsilon_0)$, $\boldsymbol{y}$ does \emph{not} belong to a ball with any center and radius $\epsilon\sqrt{n}$, with probability at least $1-\epsilon^{Cn^{\alpha/2}}$ for some constants $\epsilon_0>0$ and $C>0$. By selecting $t=\epsilon_0^2/8$, and $\epsilon=\sqrt{4t}=\epsilon_0/2$, we conclude that (\ref{eq:Fisher}) holds with probability at least $1-2A\exp(-Bt^{2+\alpha}n^{\alpha/2})-(\sqrt{4t})^{Cn^{\alpha/2}}$. This is greater than $1-a\exp(-b n^{\alpha/2})$ for some constants $a>0$ and $b>0$. Hence, $\boldsymbol{ x}_1, \boldsymbol{ x}_2, \dots,\boldsymbol{x}_M$ are Fisher-separable with probability greater than $1-\frac{M(M-1)}{2}a\exp(-b n^{\alpha/2})$. This is greater than $1-\delta$ provided that $M$ satisfies \eqref{eq:Mbound}.
\end{proof}

\begin{corollary}\label{cor:sqrtn}
Let $\{\boldsymbol{x}_1, \ldots , \boldsymbol{x}_M\}$ be a set of $M$ i.i.d. random points from an isotropic log-concave distribution in ${\mathbb R}^n$. Then set $\{\boldsymbol{x}_1, \ldots , \boldsymbol{x}_M\}$ is Fisher-separable with probability greater than $1-\delta$, $\delta>0$, provided that
$$
M \leq a c^{\sqrt{n}},
$$
where $a>0$ and $c>1$ are constants, depending only on $\delta$.
\end{corollary}
\begin{proof}
This follows from Theorem \ref{th:logconc} with $\alpha=1$ and the fact that all log-concave densities are $\psi_1$ with some universal constant, see Proposition \ref{prop:inclusions}.
\end{proof}

We say that family ${\cal P}=\{{\mathbb P}_n, \, n=1,2,\dots\}$ of probability measure has \emph{exponential Fisher separability} if there exist constants $a>0$ and $b\in(0,1)$ such that, for all $n$, inequality \eqref{eq:Fisher} holds with probability at least $1-ab^n$, where $\boldsymbol{ x}$ and $\boldsymbol{ y}$ are i.i.d vectors in ${\mathbb R}^n$ selected with respect to ${\mathbb P}_n$. In this case, for any $\delta>0$, $M$ i.i.d vectors $\{\boldsymbol{x}_1, \ldots , \boldsymbol{x}_M\}$ are Fisher-separable with probability at least $1-\delta$ provided that
$$
M \leq \sqrt{\frac{2\delta}{a}}\left(\frac{1}{\sqrt{b}}\right)^n.
$$

\begin{corollary}\label{cor:psi2}
Let ${\cal P}=\{{\mathbb P}_n, \, n=1,2,\dots\}$ be a family of isotropic log-concave probability measures 
which are all $\psi_2$ with the same constant $B_2>0$.
Then ${\cal P}$ has exponential Fisher separability.
\end{corollary}
\begin{proof}
This follows from Theorem \ref{th:logconc} with $\alpha=2$ .
\end{proof}


\begin{corollary}\label{cor:strconc}
Let ${\cal P}=\{{\mathbb P}_n, \, n=1,2,\dots\}$ be a family of isotropic probability measures 
which are all strongly log-concave with the same constant $c>0$. Then ${\cal P}$ has exponential Fisher separability.
\end{corollary}
\begin{proof}
The proof of Proposition \ref{prop:inclusions} implies that ${\mathbb P}_n$ are all $\psi_2$ with the same constant  $B_2=d/\sqrt{c}$, where $d$ is a universal constant. The statement then follows from Corollary \ref{cor:psi2}. 
\end{proof}

\begin{example}\label{cor:standnorm}
Because standard normal distribution in $\mathbb R_n$ is strongly log-concave with $c=\frac{1}{8}$, Corollary \ref{cor:strconc} implies that the family of standard normal distributions has exponential Fisher separability.
\end{example}

\subsection{Domination}

We say that family ${\cal P}'=\{{\mathbb P}'_n, \, n=1,2,\dots\}$ dominates family ${\cal P}=\{{\mathbb P}_n, \, n=1,2,\dots\}$ if there exists a constant $C$ such that 
\begin{equation}\label{eq:domin}
{\mathbb P}_n (S) \leq C \cdot {\mathbb P}'_n (S) 
\end{equation}
holds for all $n$ and all measurable subsets $S \subset {\mathbb R}^n$. In particular, if ${\mathbb P}'_n$ and ${\mathbb P}_n$ have densities $\rho'_n:{\mathbb R}^n \to [0,\infty)$ and $\rho_n:{\mathbb R}^n \to [0,\infty)$, respectively, then \eqref{eq:domin} is equivalent to 
\begin{equation}\label{eq:domindens}
\rho_n (\boldsymbol{ x}) \leq C \cdot \rho'_n (\boldsymbol{ x}), \quad \forall \boldsymbol{ x} \in {\mathbb R}^n. 
\end{equation}

\begin{theorem}\label{prop:domin}
If family ${\cal P}'$ has exponential Fisher separability, and ${\cal P}'$ dominates ${\cal P}$, then ${\cal P}$ has exponential Fisher separability. 
\end{theorem}
\begin{proof}
For every $\boldsymbol{ x}=(x_1, \dots, x_n) \in {\mathbb R}^{n}$ and $\boldsymbol{ y}= (y_1, \dots, y_n)\in {\mathbb R}^{n}$, let $\boldsymbol{ x} \times \boldsymbol{ y}$ be a point in ${\mathbb R}^{2n}$ with coordinates $(x_1, \dots, x_n, y_1, \dots, y_n)$. 
Let ${\mathbb Q}_n$ be the product measure of ${\mathbb P}_n$ with itself, that is, for every measurable set $S \subset {\mathbb R}^{2n}$, ${\mathbb Q}_n (S)$ denotes the probability that $\boldsymbol{ x} \times \boldsymbol{ y}$ belongs to $S$, where vectors $\boldsymbol{ x}$ and $\boldsymbol{ y}$ are i.i.d vectors selected with respect to ${\mathbb P}_n$.
Similarly, let ${\mathbb Q}'_n$ be the product measure of ${\mathbb P}'_n$ with itself. Inequality \eqref{eq:domin} implies that
$$
{\mathbb Q}_n (S) \leq C^2 \cdot {\mathbb Q}'_n (S), \quad \forall S \subset {\mathbb R}^{2n}.
$$
Let $A_n \subset {\mathbb R}^{2n}$ be the set of all $\boldsymbol{ x} \times \boldsymbol{ y}$ such that $(\boldsymbol{ x},\boldsymbol{ x}) \leq (\boldsymbol{ x},\boldsymbol{ y})$. Because ${\cal P}'$ has exponential Fisher separability, ${\mathbb Q}'_n (A_n) \leq ab^n$ for some $a>0$, $b\in(0,1)$. Hence, 
$$
{\mathbb Q}_n (A_n) \leq C^2 \cdot {\mathbb Q}'_n (A_n) \leq (aC^2)b^n,
$$
and exponential Fisher separability of ${\cal P}$ follows.
\end{proof}

\begin{corollary}\label{cor:domnorm}
Let ${\cal P}=\{{\mathbb P}_n, \, n=1,2,\dots\}$ be a family of distributions which is dominated by a family of (possibly scaled) standard normal distributions. Then ${\cal P}$ has exponential Fisher separability.
\end{corollary}
\begin{proof}
This follows from Example \ref{cor:standnorm}, Theorem \ref{prop:domin}, and the fact that scaling does not change Fisher separability.
\end{proof}

\section{Quasiorthogonal sets and Fisher separability of not i.i.d. data}

The technique based on concentration inequalities usually fails if the data are not identically distributed, because, in this case, each $\boldsymbol{x}_i$ may be concentrated in its \emph{own} spherical shell. An alternative approach to prove separation theorems is to use the fact that, in high dimension, almost all vectors are almost orthogonal \cite{bases}, which implies that $(\boldsymbol{x},\boldsymbol{y})$ in (\ref{eq:Fisher}) is typically ``small''. Below we apply this idea to prove Fisher separability of exponentially large families in the {``randomly perturbed''} model described in Example \ref{ex:noisy}.

Consider the {``randomly perturbed''} model from  Example \ref{ex:noisy}. In this model,  Fisher's hyperplane for separation each point $\boldsymbol{x}_i$ will be calculated assuming that coordinate center is the corresponding cluster centre $\boldsymbol{y}_i$.
\begin{theorem}\label{th:noisy}
Let $\{\boldsymbol{x}_1, \ldots , \boldsymbol{x}_M\}$ be a set of $M$ random points in the {``randomly perturbed''} model (see Example \ref{ex:noisy}) with {random perturbation} parameter $\epsilon>0$.
For any $\frac{1}{\sqrt{n}} < \delta < 1$, set $\{\boldsymbol{x}_1, \ldots , \boldsymbol{x}_M\}$ is Fisher-separable with probability at least
$$
1 - \frac{2M^2}{\delta\sqrt{n}}\left(\sqrt{1-\delta^2}\right)^{n+1}-M\left(\frac{2\delta}{\epsilon}\right)^n.
$$
In particular, set $\{\boldsymbol{x}_1, \ldots , \boldsymbol{x}_M\}$ is Fisher-separable with probability at least $1-v$, $v>0$, provided that $M<a b^n$, where $a,b$ are constants depending only on $v$ and $\epsilon$.
\end{theorem}
\begin{proof}
Let $\boldsymbol{x} \in {\mathbb R}^n$ be an arbitrary non-zero vector, and let $\boldsymbol{u}$ be a vector selected uniformly at random from a unit ball. Then, for any $\frac{1}{\sqrt{n}}<\delta<1$,
\begin{equation}\label{eq:almostort}
P\left(\left|\left(\frac{\boldsymbol{x}}{||\boldsymbol{x}||}, \frac{\boldsymbol{u}}{||\boldsymbol{u}||}\right)\right| \geq \delta \right) \leq \frac{2}{\delta\sqrt{n}}\left(\sqrt{1-\delta^2}\right)^{n+1},
\end{equation}
see \cite[Lemma 4.1]{Lovasz}.

Applying \eqref{eq:almostort} to $\boldsymbol{u}=\boldsymbol{x}_i-\boldsymbol{y}_i$, we get
$$
P\left(\left|\left(\frac{\boldsymbol{x}_j}{||\boldsymbol{x}_j||}, \frac{\boldsymbol{u}}{||\boldsymbol{u}||}\right)\right| \geq \delta \right) \leq \frac{2}{\delta\sqrt{n}}\left(\sqrt{1-\delta^2}\right)^{n+1}, \;  j\neq i,
$$
and also
$$
P\left(\left|\left(\frac{\boldsymbol{y}_i}{||\boldsymbol{y}_i||}, \frac{\boldsymbol{u}}{||\boldsymbol{u}||}\right)\right| \geq \delta \right) \leq \frac{2}{\delta\sqrt{n}}\left(\sqrt{1-\delta^2}\right)^{n+1}.
$$
On the other hand
$$
P\left(||\boldsymbol{x}_i-\boldsymbol{y}_i|| \leq 2\delta\right) = \left(\frac{2\delta}{\epsilon}\right)^n.
$$

If none of the listed events happen, then projections of all points $\boldsymbol{x}_j$, $j \neq i$, on $\boldsymbol{u}$ have length at most $\delta$ (because $||\boldsymbol{x}_j|| \leq 1, \forall j$), while the length of projection of $\boldsymbol{x}_i$ on $\boldsymbol{u}$ is greater than $\delta$, hence $\boldsymbol{x}_i$ is separable from other points by Fisher discriminant (with center $\boldsymbol{y}_i$). Hence, the probability that $\boldsymbol{x}_i$ is not separable is at most
$$
\frac{2M}{\delta\sqrt{n}}\left(\sqrt{1-\delta^2}\right)^{n+1}+\left(\frac{2\delta}{\epsilon}\right)^n
$$

The probability that there exist some index $i$ such that $\boldsymbol{x}_i$ is not separable is at most the same expression multiplied by $M$.
\end{proof}

{Theorem \ref{th:noisy} is yet another illustration of why randomization and randomized approaches to learning may improve performance of AI systems (see e.g. \cite{Wang2017},\cite{Wang2017a} for more detailed discussion on the randomized approaches and supervisory mechanisms for random parameter assignment).}

{Moreover,} Theorem \ref{th:noisy} shows that  the cluster structure of data is not an insurmountable obstacle for separation theorems. The practical experience ensures us that combination of cluster analysis with stochastic separation theorems works much better than the stochastic separation theorems directly, if there exists a pronounced cluster structure in data. The preferable way of action is:
\begin{itemize}
\item Find clusters in data clouds;
\item Create classifiers for distribution of newly coming data between clusters;
\item Apply stochastic separation theorems with discriminant (\ref{discriminant}) for each cluster separately.
\end{itemize}

This is a particular case of the general rule about complementarity between low-dimensional non-linear structures and high-dimensional stochastic separation \cite{GorTyukPhil2018}.

\section{Conclusion}

The continuous development of numerous automated AI systems for data mining is inevitable. Well-known AI products capable of responding, at least partially, to elements of the Big Data Challenge have already been developed by technological giants such as Amazon, IBM, Google, Facebook, SoftBank and many others. State-of-the art AI systems for data mining consume huge and fast-growing collections of heterogeneous data. Multiple versions of these huge-size systems have been deployed to date on millions of computers and gadgets across many various platforms. Inherent uncertainties in data result in unavoidable mistakes (e.g. mislabelling, false alarms, misdetections, wrong predictions etc.) of the AI data mining systems, which require judicious use. The successful operation of any AI system dictates that mistakes must be detected and corrected immediately and locally in the networks. However, it is prohibitively expensive and even dangerous to reconfigure big AI systems in real time.

The future development of sustainable large AI systems for mining of big data requires creation of technology and methods for fast non-iterative, non-destructive, and reversible corrections of Big Data analytics systems and for fast assimilation of new skills by the network of AI. This process should exclude human expertise as far as it is possible. In this paper we presented a brief outline of an approach to Augmented AI. This approach uses communities of interacting AI systems, fast non-destructive and non-iterative correctors of mistakes, knowledge transfer between AIs, a recommender system for distribution of problems to experts and experts to problems, and various types of audit systems. Some parts and versions of this technology have been implemented and tested.
 
 Linear Fisher's discriminant is very convenient and efficiently separates data for many distributions in high dimensions. The cascades of independent linear discriminants are also very simple and more efficient \cite{GorTyuRom2016,GorbanRomBurtTyu2016}. We have systematically tested linear and cascade correctors with simulated data and on the processing of  real videostream data \cite{GorbanRomBurtTyu2016}. The combination of low-dimensional non-linear decision rules with the  high-dimensional simple linear discriminants is a promising direction of the future development of algorithms.

New stochastic separation theorems demonstrate that the corrector technology can be used to handle errors in data flows with very general probability distributions and far away from the classical i.i.d. hypothesis.

The technology of Augmented AI is necessary for prevention of the deep failure from  the current peak of inflated interest  to intellectual solutions  (Fig.~\ref{Fig:SlideGartnerInnovationCycle}) into the gorge of deceived expectations.

\section*{Acknowledgment}

The proposed corrector methodology was implemented and successfully tested with videostream data and security tasks in collaboration with industrial partners: Apical, ARM, and VMS under support of InnovateUK. We are grateful to them and personally to I. Romanenko, R. Burton, and  K. Sofeikov.   We are grateful to M.~Gromov, who attracted our attention to the seminal question about product distributions in a multidimensional cube, and to G.~Hinton for the important remark that the typical situation with the real data flow is far from an i.i.d. sample (the points we care about seem to be from different distributions).

\end{document}